\documentclass[11pt]{article}

\usepackage[preprint]{acl}

\usepackage{times}
\usepackage{latexsym}

\usepackage[T1]{fontenc}

\usepackage[utf8]{inputenc}

\usepackage{microtype}

\usepackage{inconsolata}

\usepackage{graphicx}
\usepackage{wrapfig}
\usepackage{hyperref}
\usepackage{url}
\usepackage{booktabs}
\usepackage{graphicx}
\usepackage{subcaption}
\usepackage{mathtools}
\usepackage{amssymb}
\usepackage{multirow}
\usepackage[table]{xcolor}
\usepackage{enumitem}
\usepackage{ulem}
\usepackage{tcolorbox}
\usepackage[table]{xcolor}
\usepackage{amsthm}
\theoremstyle{plain}
\newtheorem{theorem}{Theorem}[section]
\newtheorem{lemma}[theorem]{Lemma}
\newtheorem{proposition}[theorem]{Proposition}
\newtheorem{assumptionenv}{Assumption}[section]
\theoremstyle{definition}

\theoremstyle{remark}

\usepackage{asymptote}

\usepackage{caption}
\usepackage{array} 
\usepackage{booktabs}
\usepackage{listings}
\usepackage{amssymb}
\usepackage{colortbl}      
\usepackage{graphicx}      
\definecolor{lightblue_ours}{RGB}{198,224,255}
\definecolor{lightorange_ours}{RGB}{255,229,204}
\definecolor{darkblue_ours}{RGB}{0,92,230}
\definecolor{darkorange_ours}{RGB}{230,92,0}
\newcommand{\posin}[1]{\cellcolor{lightorange_ours}{\textcolor{darkorange_ours}{#1}}}
\newcommand{\negin}[1]{\cellcolor{lightorange_ours}{#1}}
\newcommand{\posood}[1]{\cellcolor{lightblue_ours}{#1}}
\newcommand{\negood}[1]{\cellcolor{lightblue_ours}{\textcolor{darkblue_ours}{#1}}}
\newlength{\panelht}
\newcommand{\ours}{GLOW}

\title{Learning from Mistakes: Negative Reasoning Samples Enhance Out-of-Domain Generalization}

\author{
    Xueyun Tian$^{\spadesuit\heartsuit*}$, 
    Minghua Ma$^{\diamondsuit*}$, 
    Bingbing Xu$^{\spadesuit\dagger}$, 
    Nuoyan Lyu$^{\spadesuit\heartsuit}$, 
    Wei Li \\ 
    \textbf{Heng Dong}$^{\clubsuit}$, 
    \textbf{Zheng Chu}$^{\diamondsuit}$, 
    \textbf{Yuanzhuo Wang}$^{\spadesuit}$, 
    \textbf{Huawei Shen}$^{\spadesuit\heartsuit}$ \\
    $^{\spadesuit}$CAS Key Laboratory of AI Safety, Institute of Computing Technology, CAS, Beijing, China \\
    $^{\diamondsuit}$Harbin Institute of Technology, Harbin, China \\
    $^{\heartsuit}$University of Chinese Academy of Sciences, Beijing, China \\
    $^{\clubsuit}$Tsinghua University, Beijing, China \\
    \small \texttt{\{tianxueyun23z, xubingbing, lvnuoyan23z, wangyuanzhuo, shenhuawei\}@ict.ac.cn} \\
    \small \texttt{\{mhma, zchu\}@ir.hit.edu.cn} \\
    \small \texttt{weili.ucas.ict@gmail.com}, \texttt{drdhxi@gmail.com}
}

\begin{document}
\maketitle
\begingroup
  \renewcommand\thefootnote{*}
  \footnotetext{Equal contribution}
  
  \renewcommand\thefootnote{$\dagger$}
  \footnotetext{Corresponding author}
\endgroup
\begin{abstract}
Supervised fine-tuning (SFT) on chain-of-thought (CoT) trajectories demonstrations is a common approach for enabling reasoning in large language models.
Standard practices typically only retain trajectories with correct final answers (\textit{positives}) while ignoring the rest (\textit{negatives}).
We argue that this paradigm discards substantial supervision and exacerbates overfitting, limiting out-of-domain (OOD) generalization.
Specifically, we surprisingly find that incorporating \textit{negative} trajectories into SFT yields substantial OOD generalization gains over \textit{positive-only} training, as these trajectories often retain valid intermediate reasoning despite incorrect final answers.
To understand this effect in depth, we systematically analyze data, training dynamics, and inference behavior, identifying 22 recurring patterns in \textit{negative} chains that serve a dual role: they moderate loss descent to mitigate overfitting during training and boost policy entropy by 35.67\% during inference to facilitate exploration.
Motivated by these observations, we further propose \textbf{Gain-based LOss Weighting} (\ours{}), an adaptive, sample-aware scheme that exploits such distinctive training dynamics by rescaling per-sample loss based on inter-epoch progress. 
Empirically, \ours{} efficiently leverages unfiltered trajectories, yielding a 5.51\% OOD gain over \textit{positive-only} SFT on Qwen2.5-7B and boosting MMLU from 72.82\% to 76.47\% as an RL initialization.
Code is available at \href{https://github.com/Eureka-Maggie/GLOW}{Github}\footnote{\url{https://github.com/Eureka-Maggie/GLOW}}.

\end{abstract}
\section{Introduction}
\label{sec:intro}
Recent studies~\citep{yang2025qwen3technicalreport,zelikman2022star, mukherjee2023orca, shao2024deepseekmath} have established Supervised Fine-Tuning (SFT) as a foundational post-training component.
SFT adapts base models with curated instruction data, often incorporating Chain-of-Thought (CoT) trajectories to enhance reasoning capabilities. The training target typically includes the reasoning trace followed by the final answer, optimized via standard next-token prediction. The resulting model frequently serves as the initialization for subsequent reinforcement learning (RL).



However, existing SFT on distilled CoT trajectories still faces two practical limitations that compromise both effectiveness and efficiency~\citep{luo2024semi,chu2025sft,gupta2025selective,deb2025fishersft}:
(i) \textbf{Poor generalization}. 
Models may overfit to domain-specific shortcuts within demonstrations rather than acquiring transferable reasoning skills~\citep{press2022measuring,han2025general}, leading to limited transferability to out-of-distribution (OOD) tasks.
(ii) \textbf{Data inefficiency}. Current pipelines typically distill CoT trajectories from a stronger teacher and then apply rejection sampling~\citep{ahn2024large} that retains only \textit{positive} trajectories. This wastes supervision and may discard traces that contain useful intermediate reasoning signals~\citep{hamdan2025much,luo2024robustft,li2025llms}.

\begin{figure*}[t]
    \centering
    \begin{subfigure}[c]{0.65\textwidth}
        \centering
        \includegraphics[width=\linewidth]{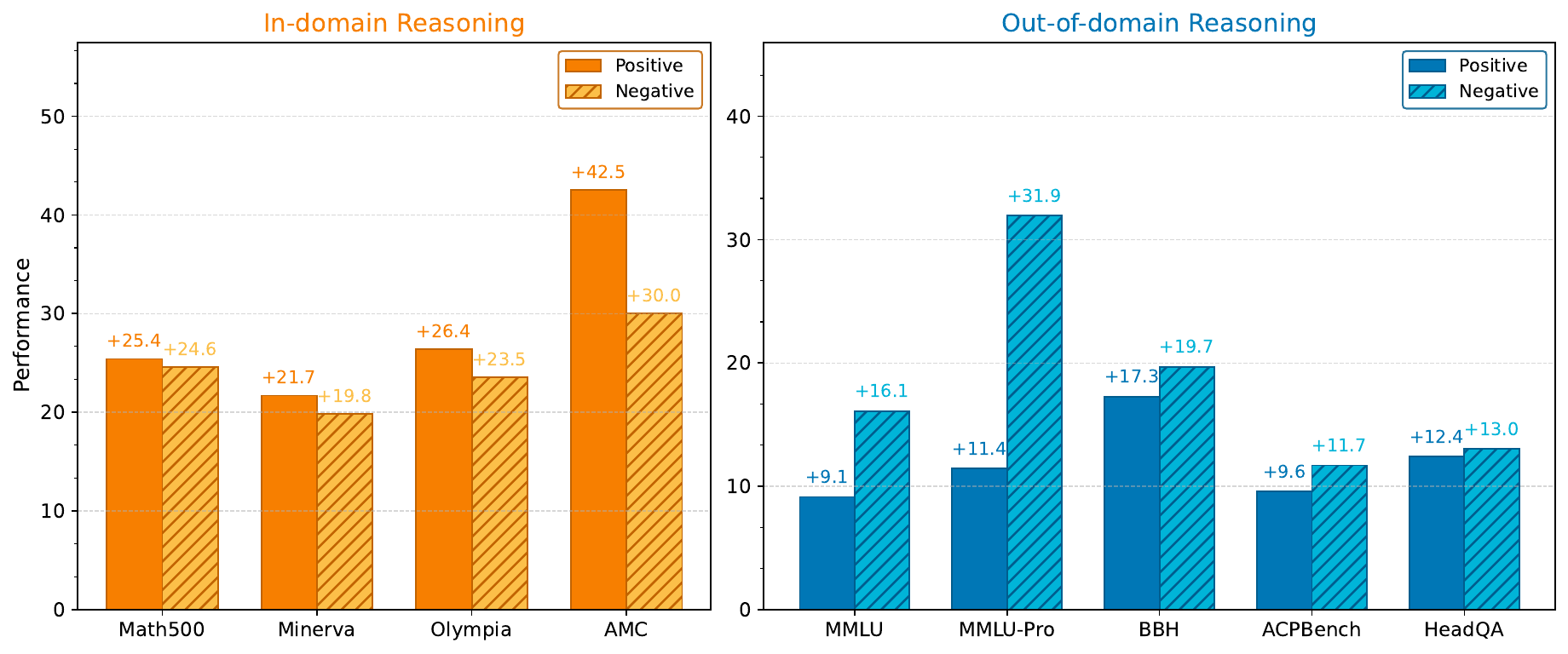}
        \caption{}
        \label{fig:intro:a}
    \end{subfigure}
    \begin{subfigure}[c]{0.34\textwidth}
        \centering
        \includegraphics[width=\linewidth]{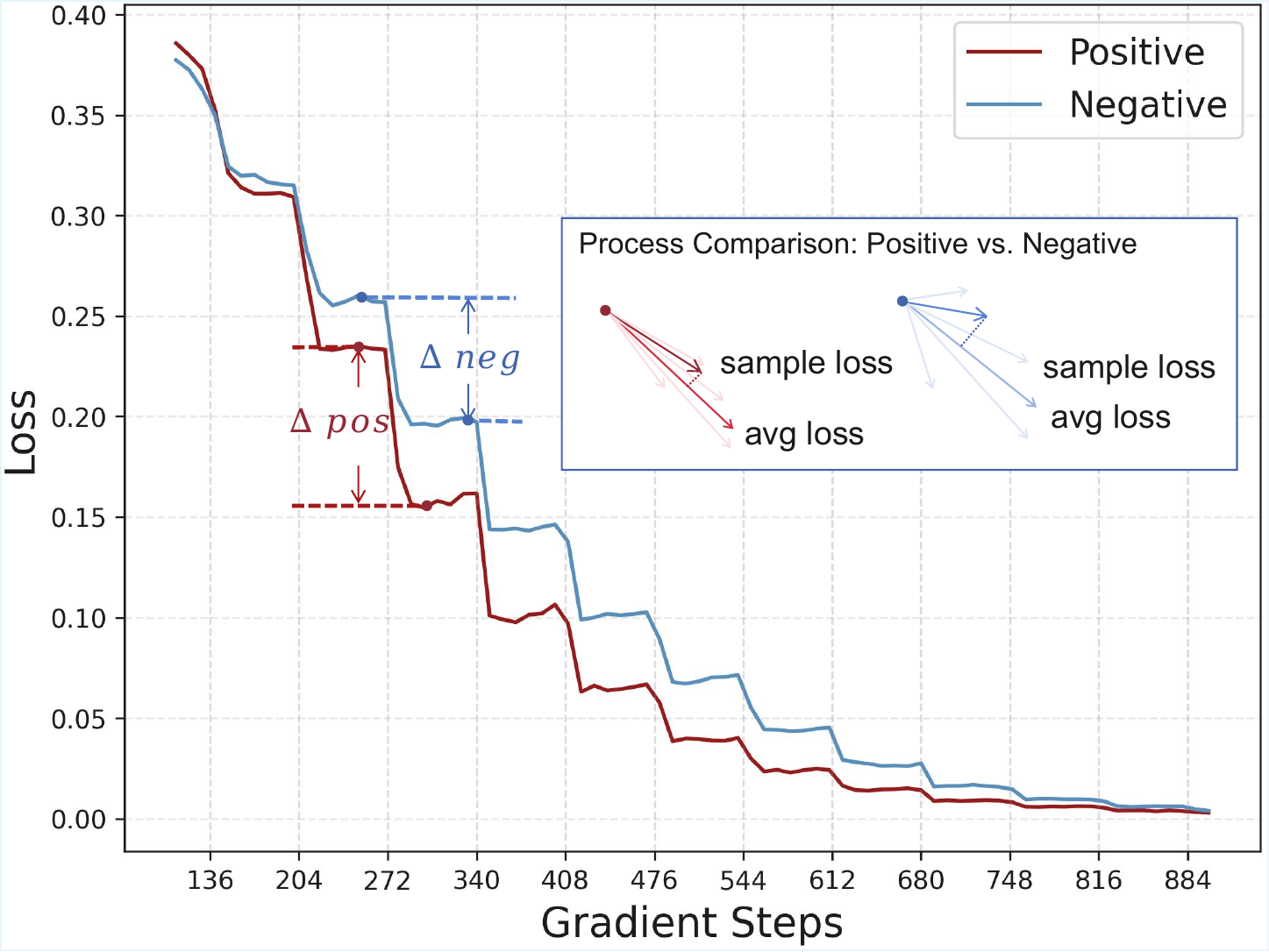}
        \caption{}
        \label{fig:intro:b}
    \end{subfigure}
    \caption{
    (a) Qwen2.5-14B: SFT on \textit{positives} improves in-domain math but transfers weakly to other reasoning tasks, whereas SFT on \textit{negatives} yields broader cross-domain gains. Bars show final accuracy, and “+” indicates absolute improvement over the base model.
    (b) Qwen2.5-32B: training loss on MMLU. Red denotes \textit{positive-only} SFT and blue denotes \textit{negative-only} SFT. $\Delta$ is the per-sample inter-epoch loss difference.}
    \label{fig:intro}
    \vspace{-13pt}
\end{figure*}

We argue that these typically discarded \textit{negatives} offer a promising opportunity to alleviate both limitations, as these often include valid intermediate reasoning and diverse modes.
To investigate this, 
we distill math reasoning trajectories from Qwen3-8B~\citep{yang2025qwen3} and compare student models trained only on \textit{positives} versus only on \textit{negatives}. Figure~\ref{fig:intro:a} shows a surprising result: \textbf{models trained only on \textit{negatives} outperform those trained only on \textit{positives} on many benchmarks, with larger gains on OOD evaluations.}

This counterintuitive effect motivates a deeper analysis of \textit{negatives} across data, optimization, and inference.
\textbf{Regarding data}, we identify 9 error types with 22 recurring patterns (Table~\ref{tab:error_category}). This diversity exposes the model to broad error regimes, fostering intrinsic reasoning signals that generalize across contexts.
\textbf{In terms of optimization}, \textit{negative-only} SFT shows slower convergence yet steady performance gains across epochs (Figure~\ref{fig:intro:b}, Table~\ref{tab:epoch_sweep_pos_neg}). The consistently smaller inter-epoch loss reduction ($\Delta$) implies a more challenging optimization landscape that resists rapid convergence, thereby mitigating shortcut overfitting and compelling the model to learn robust reasoning features rather than spurious correlations.
\textbf{For inference}, training on \textit{negatives} significantly boosts policy entropy and pass@k on OOD tasks (Appendix~\ref{app:passk}), thereby facilitating diverse exploration and enhancing generalization, respectively.
\textbf{
Overall, these insights reveal a cohesive mechanism: the diverse patterns in \textit{negatives} act as a natural regularizer that modulates training dynamics to prevent shortcut learning while increasing inference entropy to encourage exploration, collectively unlocking superior OOD generalization.
}

Motivated by these observations, we seek to synergize the strengths of \textit{positive} and \textit{negative} trajectories within the SFT framework.
To achieve this, we propose \textbf{G}ain-based \textbf{LO}ss \textbf{W}eighting (\ours{}), a dynamic reweighting scheme utilizing the entire dataset to maximize sample efficiency without explicit filtering.
During training, \ours{} measures each sample’s gain as its inter-epoch loss reduction and adaptively upweights those with low gain.
Such samples, typically aligning with the \textit{negatives} with small $\Delta$, signal insufficient learning and steer optimization toward undercovered reasoning patterns.
Empirically, \ours{} yields consistent gains across model families and scales: on Qwen2.5-7B, it improves average performance by 2.14\% over mixed-data training and OOD performance by 5.51\% over \textit{positive-only} SFT, and as an RL initialization, it further boosts MMLU from 72.82\% to 76.47\% under the same RL setup (Table~\ref{tab:rl_init}).

Our contributions can be summarized as follows:
\begin{itemize}
    \item \textbf{Systematic investigation of \textit{negatives}:} We demonstrate negative trajectories significantly enhance OOD generalization. A unified analysis across data, optimization, and inference reveals that exposure to diverse error patterns mitigates overfitting and fosters exploration.
    \item \textbf{Adaptive Training Strategy:} We propose a sample-aware reweighting strategy for utilizing unfiltered data. By modulating loss based on inter-epoch learning progress, \ours{} prioritizes underexplored patterns, enabling efficient and generalizable SFT.
    \item \textbf{Superior SFT Generalization and RL Initialization:} Experiments validate \ours{} across diverse benchmarks. It yields consistent OOD improvements and transfers effectively to RL, serving as a superior initialization that amplifies the gains from RL. 
    
\end{itemize}

\section{Related Works}
\paragraph{Supervised Fine-Tuning for Reasoning}
SFT is a widely adopted approach for strengthening the reasoning ability of large language models~\citep{wei2021finetuned,ouyang2022training}. A common recipe is that we distills CoT trajectories from stronger teacher models and uses them to supervise smaller or less capable students~\citep{shao2024deepseekmathpushinglimitsmathematical,zheng2025groupsequencepolicyoptimization,yu2025dapoopensourcellmreinforcement}. To ensure data quality, conventional pipelines often employ rejection sampling as a rigorous filter, retaining only those trajectories that yield correct final answers~\citep{ahn2024large}. Such CoT-based SFT can transfer long-form reasoning patterns and often provides a strong initialization for subsequent reinforcement learning~\citep{lewkowycz2022solving,shao2024deepseekmath}. However, this heavy filtering discards a substantial portion of available trajectories, wastefully discarding rich supervisory information.

\paragraph{Learning from Negative Data}
Prior work leverages negative samples mainly in three ways: prompting, fine-tuning, and reinforcement learning. Prompt-based methods place negative examples in the context to steer generation away from undesired behaviors~\citep{gao2024customizing,alazraki2025no}. Their effectiveness, however, depends on the model’s existing instruction-following and reasoning ability, which limits their impact on weak students.
Fine-tuning-based approaches use negative data more indirectly. A common strategy is to convert initially incorrect trajectories into positive CoT supervision via teacher rewriting or refinement~\citep{yu2025self,pan2025lemma,an2023learning}. Other works add explicit markers or prefixes to separate correct from incorrect samples during training~\citep{wang2024learning,tong2024can}. These methods do not establish whether learning from raw incorrect trajectories themselves improves generalization.

\paragraph{Domain Generalization in LLMs}
Most fine-tuning work improves reasoning within a single domain, such as mathematics or code, while cross-domain transfer remains underexplored. \citet{huan2025does} show that SFT on math induces substantial representation shifts that can degrade general capabilities. \citet{wu2025knowledge} propose knowledge index and information gain to separate knowledge from reasoning, and find that SFT on math offers limited benefit in knowledge-intensive domains. \citet{yang2025decoupling} and \citet{zhao2025chain} further argue that SFT often learns superficial reasoning traces and transfers poorly across domains. These studies are primarily diagnostic and do not develop methods or examine how data selection and supervision signals affect cross-domain generalization.

\section{The Surprising Phenomenon: Negatives Generalize Better}
\label{sec:phenomenon}
In this section, we describe the empirical phenomenon that motivates our study: fine-tuning on negative reasoning samples can enhance OOD generalization more effectively than fine-tuning on \textit{positive} samples. We first detail the controlled experiments designed to validate this phenomenon and then present results that demonstrate its consistency across diverse benchmarks and model scales.

\subsection{Data Construction and Training Setup}
\label{sec:exp_setting}
Using Qwen3-8B, we distill trajectories from OpenMathReasoning~\citep{moshkov2025aimo2} and MMLU~\citep{hendryckstest2021}, labeling those matching the ground truth as \textit{positive} and others as \textit{negative}. We construct balanced datasets of complete reasoning chains to fine-tune Qwen-2.5 (from 3B to 32B) and Llama-3.1 8B. See Appendix~\ref{app:dataset} for detailed configurations.

\subsection{Negatives Surpass Positives in OOD}

\label{subsec:results_of_phenomenon}

As shown in Table~\ref{tab:full_math} and Table~\ref{tab:full_mmlu}, we surprisingly find that training on \textit{negative} samples, although it yields smaller improvements than \textit{positive} samples on in-domain performance, consistently improves OOD generalization. Overall, models trained on \textit{negative} math reasoning samples achieve an average improvement of 11.97\% on general reasoning tasks and 4.11\% on other reasoning tasks. Similarly, models trained on \textit{negative} MMLU samples gain an average of 1.98\% on mathematical reasoning and 1.35\% on other reasoning benchmarks. Although mathematical problems are generally more suitable for constructing reasoning-focused data, the same trend is observed for models trained on MMLU, indicating that the benefit of \textit{negative} samples for OOD generalization is not limited to a specific domain. These observations motivate a deeper analysis into the underlying factors that make \textit{negative} samples more effective for enhancing OOD reasoning performance.



\begin{table*}[htbp]
\centering
\resizebox{\textwidth}{!}{
\begin{tabular}{ll|cccc|c|ccc|c|cc|c}
\toprule
\multicolumn{2}{c|}{ } 
 & \multicolumn{5}{c|}{Math Reasoning (In-Domain)} 
 & \multicolumn{4}{c|}{General Reasoning (Out-of-Domain)} 
 & \multicolumn{3}{c}{Other Reasoning (Out-of-Domain)} \\
\cmidrule(lr){3-7} \cmidrule(lr){8-11} \cmidrule(lr){12-14}
Model & Setting 
 & Math500 & Minerva & Olympia & AMC & \multicolumn{1}{c|}{Avg.} 
 & MMLU & MMLU-Pro & BBH & \multicolumn{1}{c|}{Avg.} 
 & ACPBench & HeadQA & Avg. \\
\specialrule{0.65pt}{0pt}{0pt}
\multirow{5}{*}{Qwen2.5-3B} 
 & Base        & 52.60 & 21.32 & 22.52 & 32.50 & 32.24 
                & 31.88 & 12.54 & 27.75 & 24.06 
                & 23.31 & 33.15 & 28.23 \\
 & Full        & 60.80 & 26.10 & 23.26 & 35.00 & 36.29
                & 64.13 & 38.66 & 52.29 & 51.69 
                & 32.68 & 62.69 & 47.69 \\
\cmidrule(lr){2-14}
 & Positive    & \posin{61.60} & \posin{25.74} & \posin{24.44} & \posin{42.50} & \posin{38.60} 
                & \posood{54.45} & \posood{25.62} & \posood{44.35} & \posood{41.50} 
                & \posood{30.21} & \posood{59.81} & \negood{45.01} \\
 & Negative    & \negin{58.60} & \negin{23.53} & \negin{24.15} & \negin{42.50} & \negin{37.20} 
                & \negood{64.09} & \negood{39.20} & \negood{53.87} & \negood{52.39} 
                & \negood{33.06} & \negood{63.13} & \posood{48.10} \\
 & \cellcolor{gray!20}{$\Delta(\text{pos-neg})$} 
                & \cellcolor{gray!20}{+3.00} & \cellcolor{gray!20}{+2.21} & \cellcolor{gray!20}{+0.29} & \cellcolor{gray!20}{0.00} & \cellcolor{gray!20}{+1.38} 
                & \cellcolor{gray!20}{-9.64} & \cellcolor{gray!20}{-13.58} & \cellcolor{gray!20}{-9.52} & \cellcolor{gray!20}{-10.91} 
                & \cellcolor{gray!20}{-2.85} & \cellcolor{gray!20}{-3.32} & \cellcolor{gray!20}{-3.09} \\
\midrule

\multirow{5}{*}{Qwen2.5-7B} 
 & Base        & 58.40 & 26.84 & 26.07 & 52.50 & 40.95 
                & 55.80 & 26.56 & 51.10 & 44.49 
                & 28.77 & 57.29 & 43.03 \\
 & Full        & 76.60 & 40.07 & 38.96 & 55.00 & 52.66 
                & 72.24 & 53.71 & 70.84 & 65.60 
                & 38.27 & 72.06 & 55.17 \\
\cmidrule(lr){2-14}
 & Positive    & \posin{78.00} & 36.76 & \posin{41.78} & \posin{57.50} & \posin{53.51} 
                & \posood{61.03} & \posood{32.70} & \posood{60.58} & \posood{51.44} 
                & \posood{33.38} & \posood{68.60} & \posood{50.99} \\
 & Negative    & \negin{77.60} & 40.44 & \negin{38.37} & \negin{57.50} & \negin{53.48} 
                & \negood{73.11} & \negood{53.74} & \negood{71.73} & \negood{66.19} 
                & \negood{38.98} & \negood{71.81} & \posood{55.40} \\
 & \cellcolor{gray!20}{$\Delta$(pos-neg)}
                & \cellcolor{gray!20}{+0.40} & \cellcolor{gray!20}{-3.68} & \cellcolor{gray!20}{+3.41} & \cellcolor{gray!20}{0.00} & \cellcolor{gray!20}{+0.03} 
                & \cellcolor{gray!20}{-12.08} & \cellcolor{gray!20}{-21.04} & \cellcolor{gray!20}{-11.15} & \cellcolor{gray!20}{-14.76} 
                & \cellcolor{gray!20}{-5.60} & \cellcolor{gray!20}{-3.21} & \cellcolor{gray!20}{-4.41} \\
\midrule

\multirow{5}{*}{Qwen2.5-14B} 
 & Base        & 62.60 & 26.84 & 27.56 & 40.00 & 39.25 
                & 64.68 & 35.77 & 59.27 & 53.24 
                & 37.04 & 68.75 & 52.90 \\
 & Full        & 86.80 & 47.79 & 52.30 & 82.50 & 67.35 
                & 81.56 & 67.63 & 80.90 & 76.70 
                & 48.13 & 81.44 & 64.79 \\
\cmidrule(lr){2-14}
 & Positive    & \posin{88.00} & \posin{48.53} & \posin{53.93} & \posin{82.50} & \posin{68.24} 
                & \posood{73.81} & \posood{47.21} & \posood{76.54} & \posood{65.85} 
                & \posood{46.62} & \posood{81.15} & \posood{63.89} \\
 & Negative    & \negin{87.20} & \negin{46.69} & \negin{51.11} & \negin{70.00} & \negin{63.75} 
                & \negood{80.77} & \negood{67.70} & \negood{78.95} & \negood{75.81} 
                & \negood{48.73} & \negood{81.77} & \negood{65.25} \\
 & \cellcolor{gray!20}{$\Delta$(pos-neg)} 
                & \cellcolor{gray!20}{+0.80} & \cellcolor{gray!20}{+1.84} & \cellcolor{gray!20}{+2.82} & \cellcolor{gray!20}{+12.50} & \cellcolor{gray!20}{+4.49} 
                & \cellcolor{gray!20}{-6.96} & \cellcolor{gray!20}{-20.49} & \cellcolor{gray!20}{-2.41} & \cellcolor{gray!20}{-9.95} 
                & \cellcolor{gray!20}{-2.11} & \cellcolor{gray!20}{-0.62} & \cellcolor{gray!20}{-1.37} \\
\midrule

\multirow{5}{*}{Qwen2.5-32B} 
 & Base        & 63.20 & 34.19 & 26.52 & 35.00 & 39.73 
                & 68.34 & 39.80 & 58.65 & 55.60 
                & 38.63 & 68.45 & 53.54 \\
 & Full        & 92.20 & 52.57 & 57.19 & 85.00 & 71.74 
                & 85.22 & 73.10 & 83.53 & 80.62 
                & 50.67 & 84.90 & 67.79 \\
\cmidrule(lr){2-14}
 & Positive    & 91.40 & \posin{50.74} & \posin{60.89} & 85.00 & 72.01 
                & \posood{79.01} & \posood{54.31} & \posood{80.61} & \posood{71.31} 
                & \posood{49.96} & \posood{83.15} & \posood{66.56} \\
 & Negative    & 92.20 & \negin{50.74} & \negin{58.37} & 95.00 & 74.08 
                & \negood{85.47} & \negood{73.53} & \negood{84.51} & \negood{81.17} 
                & \negood{51.80} & \negood{85.27} & \negood{68.54} \\
 & \cellcolor{gray!20}{$\Delta$(pos-neg)} 
                & \cellcolor{gray!20}{-0.80} & \cellcolor{gray!20}{0.00} & \cellcolor{gray!20}{+2.52} & \cellcolor{gray!20}{-10.00} & \cellcolor{gray!20}{-2.07} 
                & \cellcolor{gray!20}{-6.46} & \cellcolor{gray!20}{-19.22} & \cellcolor{gray!20}{-3.90} & \cellcolor{gray!20}{-9.86} 
                & \cellcolor{gray!20}{-1.84} & \cellcolor{gray!20}{-2.12} & \cellcolor{gray!20}{-1.98} \\
\midrule

\multirow{5}{*}{Llama3.1-8B} 
 & Base        & 2.80 & 1.10 & 0.44 & 0.00 & 1.09 
                & 66.49 & 0.47 & 2.33 & 23.10 
                & 5.18 & 2.30 & 3.74 \\
 & Full        & 41.20 & 18.01 & 14.67 & 15.00 & 22.22 
                & 62.48 & 36.88 & 55.12 & 51.49 
                & 32.96 & 65.90 & 49.43 \\
\cmidrule(lr){2-14}
 & Positive    & \posin{37.80} & 18.01 & \posin{10.37} & \posin{12.50} & 19.67 
                & \posood{41.95} & \posood{23.15} & \posood{45.07} & \posood{36.72} 
                & \posood{31.20} & \posood{47.81} & \posood{39.50} \\
 & Negative    & \negin{34.40} & 18.38 & \negin{9.19} & \negin{20.00} & 20.49 
                & \negood{62.14} & \negood{36.22} & \negood{54.85} & \negood{51.07} 
                & \negood{33.31} & \negood{65.17} & \negood{49.24} \\
 & \cellcolor{gray!20}{$\Delta$(pos-neg)} 
                & \cellcolor{gray!20}{+3.40} & \cellcolor{gray!20}{-0.37} & \cellcolor{gray!20}{+1.18} & \cellcolor{gray!20}{-7.50} & \cellcolor{gray!20}{-0.82} 
                & \cellcolor{gray!20}{-20.19} & \cellcolor{gray!20}{-13.07} & \cellcolor{gray!20}{-9.78} & \cellcolor{gray!20}{-14.35} 
                & \cellcolor{gray!20}{-2.11} & \cellcolor{gray!20}{-17.36} & \cellcolor{gray!20}{-9.74} \\
\bottomrule
\end{tabular}
}
\caption{Cross-domain performance on \textbf{math reasoning}. ``Avg.'' is the within-group average. \posin{orange} highlights in-domain benchmarks where \textit{positives} outperform \textit{negatives}, and \negood{blue} highlights OOD benchmarks where \textit{negatives} outperform \textit{positives}. The higher score in each pair is highlighted accordingly.
}
\label{tab:full_math}
\vspace{-10pt}
\end{table*}

\begin{table*}[htbp]
\vspace{-10pt}
\centering

\resizebox{\textwidth}{!}{
\begin{tabular}{ll|cccc|c|ccc|c|cc|c}
\toprule
\multicolumn{2}{c|}{ } 
 & \multicolumn{5}{c|}{Math Reasoning (Out-of-Domain)} 
 & \multicolumn{4}{c|}{General Reasoning (In-Domain)} 
 & \multicolumn{3}{c}{Other Reasoning (Out-of-Domain)} \\
\cmidrule(lr){3-7} \cmidrule(lr){8-11} \cmidrule(lr){12-14}
Model & Setting 
 & Math500 & Minerva & Olympia & AMC & \multicolumn{1}{c|}{Avg.} 
 & MMLU & MMLU-Pro & BBH & \multicolumn{1}{c|}{Avg.} 
 & ACPBench & HeadQA & Avg. \\
\specialrule{0.65pt}{0pt}{0pt}
\multirow{5}{*}{Qwen2.5-3B} 
 & Base        & 52.60 & 21.32 & 22.52 & 32.50 & 32.24 
                & 31.88 & 12.54 & 27.75 & 24.06 
                & 23.31 & 33.15 & 28.23 \\
 & Full        & 58.20 & 23.16 & 25.19 & 35.00 & 35.39 
                & 66.74 & 40.82 & 53.35 & 53.64 
                & 35.70 & 67.61 & 51.66 \\
\cmidrule(lr){2-14}
 & Positive    & \posood{59.20} & \posood{27.21} & \posood{25.04} & \posood{30.00} & \posood{35.36} 
                & \posin{67.88} & \posin{42.56} & \posin{52.84} & \posin{54.43} 
                & \posood{34.93} & \posood{67.69} & \posood{51.31} \\
 & Negative    & \negood{59.60} & \negood{28.31} & \negood{25.48} & \negood{40.00} & \negood{38.35} 
                & \negin{65.42} & \negin{38.55} & \negin{52.28} & \negin{52.08} 
                & \negood{36.13} & \negood{68.85} & \negood{52.49} \\
 & \cellcolor{gray!20}{$\Delta$(pos-neg)} 
                & \cellcolor{gray!20}{-0.40} & \cellcolor{gray!20}{-1.10} & \cellcolor{gray!20}{-0.44} & \cellcolor{gray!20}{-10.00} & \cellcolor{gray!20}{-2.99} 
                & \cellcolor{gray!20}{+2.32} & \cellcolor{gray!20}{+4.01} & \cellcolor{gray!20}{+0.56} & \cellcolor{gray!20}{+2.30} 
                & \cellcolor{gray!20}{-1.20} & \cellcolor{gray!20}{-1.16} & \cellcolor{gray!20}{-1.18} \\
\midrule

\multirow{4}{*}{Qwen2.5-7B} 
 & Base        & 58.40 & 26.84 & 26.07 & 52.50 & 40.95 
                & 55.80 & 26.56 & 51.10 & 44.49 
                & 28.77 & 57.29 & 43.03 \\
 & Full        & 75.60 & 38.60 & 40.15 & 47.50 & 50.46 
                & 73.14 & 51.15 & 71.30 & 65.20 
                & 42.18 & 72.76 & 57.47 \\
\cmidrule(lr){2-14}
 & Positive    & \posood{74.40} & 37.50 & \posood{39.85} & \posood{50.00} & \posood{50.44} 
                & \posin{73.42} & \posin{53.22} & 68.23 & \posin{64.96} 
                & \posood{40.32} & 74.25 & \posood{57.29} \\
 & Negative    & \negood{77.00} & 37.13 & \negood{42.07} & \negood{60.00} & \negood{54.05} 
                & \negin{71.23} & \negin{45.79} & 69.46 & \negin{62.16} 
                & \negood{42.61} & 73.38 & \negood{58.00} \\
 & \cellcolor{gray!20}{$\Delta$(pos-neg)} 
        & \cellcolor{gray!20}{-2.60} & \cellcolor{gray!20}{+0.37} & \cellcolor{gray!20}{-2.22} & \cellcolor{gray!20}{-10.00} & \cellcolor{gray!20}{-3.61} 
        & \cellcolor{gray!20}{+2.19} & \cellcolor{gray!20}{+7.43} & \cellcolor{gray!20}{-1.23} & \cellcolor{gray!20}{+2.80} 
        & \cellcolor{gray!20}{-2.29} & \cellcolor{gray!20}{+0.87} & \cellcolor{gray!20}{-0.71} \\

\midrule
\multirow{4}{*}{Qwen2.5-14B} 
 & Base        & 62.60 & 26.84 & 27.56 & 40.00 & 39.25 
                & 64.68 & 35.77 & 59.27 & 53.24 
                & 37.04 & 68.75 & 52.90 \\
 & Full        & 82.20 & 43.01 & 51.85 & 70.00 & 61.77 
                & 78.13 & 59.57 & 80.56 & 72.75 
                & 48.87 & 79.94 & 64.41 \\
\cmidrule(lr){2-14}
 & Positive    & \posood{80.20} & \posood{42.28} & 50.96 & 72.50 & 61.49 
                & \posin{80.09} & \posin{65.26} & \posin{80.21} & \posin{75.19} 
                & 48.56 & \posood{80.53} & 64.55 \\
 & Negative    & \negood{83.00} & \negood{45.22} & 48.89 & 65.00 & 60.53 
                & \negin{76.83} & \negin{56.03} & \negin{80.15} & \negin{71.00} 
                & 48.27 & \negood{80.56} & 64.42 \\
 & \cellcolor{gray!20}{$\Delta$(pos-neg)} 
        & \cellcolor{gray!20}{-2.80} & \cellcolor{gray!20}{-2.94} & \cellcolor{gray!20}{+2.07} & \cellcolor{gray!20}{+7.50} & \cellcolor{gray!20}{+0.96} 
        & \cellcolor{gray!20}{+3.26} & \cellcolor{gray!20}{+9.23} & \cellcolor{gray!20}{+0.06} & \cellcolor{gray!20}{+4.18} 
        & \cellcolor{gray!20}{+0.29} & \cellcolor{gray!20}{-0.03} & \cellcolor{gray!20}{+0.13} \\

\midrule
\multirow{4}{*}{Qwen2.5-32B} 
 & Base        & 63.20 & 34.19 & 26.52 & 35.00 & 39.73 
                & 68.34 & 39.80 & 58.65 & 55.60 
                & 38.63 & 68.45 & 53.54 \\
 & Full        & 86.60 & 46.69 & 55.70 & 80.00 & 67.25 
                & 79.06 & 61.15 & 79.94 & 73.38 
                & 49.89 & 83.01 & 66.45 \\
\cmidrule(lr){2-14}
 & Positive    & \posood{85.20} & \posood{46.69} & \posood{56.15} & 75.00 & 65.76 
                & \posin{81.97} & \posin{68.54} & \posin{81.60} & \posin{77.37} 
                & \posood{50.35} & 82.90 & 66.63 \\
 & Negative    & \negood{86.40} & \negood{47.06} & \negood{56.89} & 72.50 & 65.71 
                & \negin{77.99} & \negin{58.34} & \negin{80.71} & \negin{72.35} 
                & \negood{51.20} & 82.39 & \negood{66.80} \\
 & \cellcolor{gray!20}{$\Delta$(pos-neg)} 
        & \cellcolor{gray!20}{-1.20} & \cellcolor{gray!20}{-0.37} & \cellcolor{gray!20}{-0.74} & \cellcolor{gray!20}{+2.50} & \cellcolor{gray!20}{+0.05} 
        & \cellcolor{gray!20}{+3.98} & \cellcolor{gray!20}{+10.20} & \cellcolor{gray!20}{+0.89} & \cellcolor{gray!20}{+5.02} 
        & \cellcolor{gray!20}{-0.85} & \cellcolor{gray!20}{+0.51} & \cellcolor{gray!20}{-0.17} \\

\midrule
\multirow{4}{*}{Llama3.1-8B} 
 & Base        & 2.80 & 1.10 & 0.44 & 0.00 & 1.09 
                & 66.49 & 0.47 & 2.33 & 23.10 
                & 5.18 & 2.30 & 3.74 \\
 & Full        & 20.00 & 15.81 & 6.52 & 2.50 & 11.21 
                & 66.49 & 40.56 & 53.73 & 53.59 
                & 36.06 & 69.55 & 52.81 \\
\cmidrule(lr){2-14}
 & Positive    & \posood{15.60} & \posood{11.76} & \posood{3.85} & \posood{7.50} & \posood{9.68} 
                & \posin{64.73} & \posin{39.74} & 45.39 & 49.95 
                & \posood{29.61} & \posood{67.69} & \posood{48.65} \\
 & Negative    & \negood{23.00} & \negood{16.18} & \negood{6.67} & \negood{10.00} & \negood{13.96} 
                & \negin{64.63} & \negin{38.85} & 53.23 & 52.24 
                & \negood{37.15} & \negood{69.80} & \negood{53.48} \\
 & \cellcolor{gray!20}{$\Delta$(pos-neg)} 
        & \cellcolor{gray!20}{-7.40} & \cellcolor{gray!20}{-4.42} & \cellcolor{gray!20}{-2.82} & \cellcolor{gray!20}{-2.50} & \cellcolor{gray!20}{-4.29} 
        & \cellcolor{gray!20}{+0.10} & \cellcolor{gray!20}{+0.89} & \cellcolor{gray!20}{-7.84} & \cellcolor{gray!20}{-2.28} 
        & \cellcolor{gray!20}{-7.54} & \cellcolor{gray!20}{-2.11} & \cellcolor{gray!20}{-4.83} \\

\bottomrule
\end{tabular}
}
\caption{Cross-domain performance on \textbf{general reasoning}. ``Avg.'' is the within-group average. \posin{orange} highlights in-domain benchmarks where \textit{positives} outperform \textit{negatives}, and \negood{blue} highlights OOD benchmarks where \textit{negatives} outperform \textit{positives}. The higher score in each pair is highlighted accordingly.
}
\label{tab:full_mmlu}
\end{table*}

\section{Why Negative is Better}
\label{sec:analysis}
To explain why \textit{negatives} benefit OOD generalization, we analyze the phenomenon from data, optimization, and inference perspectives. Empirically, \textit{positives} tend to share a small set of success pattern, while \textit{negatives} exhibit much richer failure modes. We first characterize the diversity introduced by \textit{negatives}. We then examine training dynamics to show how this diversity shapes optimization. Finally, we analyze inference behavior to connect these effects to improved OOD performance.


\subsection{Data Perspective}
\label{sec:analysis_data}
\begin{table}[t]
\centering
\renewcommand{\arraystretch}{0.92}
\resizebox{\columnwidth}{!}{%
\begin{tabular}{lcc}
    \toprule
    \textbf{Error Categories} & \textbf{OpenMathReasoning} & \textbf{MMLU} \\
    \midrule
    Calculation      & 27   & 9    \\
    Completeness     & 11   & 28   \\
    Evaluation System& 2599 & 2024 \\
    Formal           & 57   & 123  \\
    Knowledge        & 27   & 199  \\
    Logical          & 195  & 4116 \\
    Programming      & 8    & 5    \\
    Understanding    & 435  & 1056 \\
    Special Cases    & 301  & 1137 \\
    \midrule
    \textbf{Total}   & \textbf{3660} & \textbf{8697} \\
    \bottomrule
\end{tabular}%
}
\caption{Error categorization in the negative OpenMathReasoning and MMLU samples.}
\label{tab:error_category}
\vspace{-12pt}
\end{table}

Following~\citep{he2025can}, we observe that reasoning errors manifest in 9 major types and 22 subtypes. For each \textit{negative} trajectory in OpenMathReasoning and the MMLU training set, we use Gemini-2.5-Pro~\citep{comanici2025gemini} to assign an error label (the prompt is in Appendix~\ref{app:prompt}). Table~\ref{tab:error_category} shows a broad and diverse distribution that spans logical mistakes, comprehension errors, and other failure modes. This diversity implies that \textit{negatives} cover substantially more heterogeneous reasoning patterns than \textit{positives}, which tend to follow more uniform solution templates. The full label definitions are provided in Appendix~\ref{app:cat}.

\textbf{\textit{Negatives} improve OOD generalization by exposing the model to diverse error regimes, which encourages invariant reasoning features.}
We view error categories in \textit{negatives} as environments in the sense of IRM~\citep{arjovsky2019invariant} (formalized in Appendix~\ref{app:irm_view}), where generalization benefits from signals that remain stable across heterogeneous environments. Each error type defines a distinct failure regime. \textit{Negatives} are not pure noise, since many trajectories contain partially valid reasoning segments (Figure~\ref{fig:case_study_neg_path}), and performance continues to improve over epochs when training on \textit{negatives} (Table~\ref{tab:epoch_sweep_pos_neg}). This diversity compels the model to learn invariant features stable across distinct regimes, whereas positives cover fewer paths and offer weaker incentives for such stability.
\subsection{Training Perspective}
\label{sec:analysis_train}
To characterize learning dynamics, we log training loss every 10 steps for models fine-tuned on \textit{positives} and \textit{negatives} from math reasoning and MMLU. We use Qwen2.5-32B as a representative example (Figure~\ref{fig:intro:b}) with additional training curves are deferred to Appendix~\ref{app:loss}. Across settings, the loss exhibits a consistent stage-wise pattern. With \textit{positives}, loss drops abruptly near epoch boundaries and converges faster early on. With \textit{negatives}, loss decreases more smoothly and gradually, yet converges to a comparable level.

We attribute the loss disparity to signal diversity: homogeneous \textit{positives} drive rapid early drops via redundant updates, whereas heterogeneous \textit{negatives} induce steadier, broader progress. Table~\ref{tab:avg-loss-delta} confirms this early gap ($\Delta_{\text{pos}} > \Delta_{\text{neg}}$). This sustained descent reflects reduced shortcut fitting, aligning with the superior OOD generalization of \textit{negatives}.

Importantly, loss on \textit{negatives} keeps decreasing throughout training (Figures~\ref{fig:intro:b} and~\ref{fig:all}) and is accompanied by steady gains on benchmarks at multiple training checkpoints (Table~\ref{tab:epoch_sweep_pos_neg} and Appendix~\ref{app:progress}). This indicates that \textit{negatives} provide learnable supervision rather than noise. They combine partially valid reasoning with diverse failure patterns, yielding sustained training signals and promoting robust reasoning over memorizing narrow solution templates.


\textbf{Overall, these results indicate that the training value of \textit{negatives} lies in their diversity: they slow early loss descent while providing heterogeneous optimization signals that broaden reasoning patterns and improve OOD generalization.}

\begin{table}[htbp]
\centering

\resizebox{\columnwidth}{!}{%
\begin{tabular}{lcccc}
\toprule
\multicolumn{1}{c}{Model} &
\multicolumn{1}{c}{$\Delta_{\text{avg\_loss}}^{\text{epoch }2-1}$} &
\multicolumn{1}{c}{$\Delta_{\text{avg\_loss}}^{\text{epoch }3-2}$} &
\multicolumn{1}{c}{$\Delta_{\text{avg\_loss}}^{\text{epoch }4-3}$} &
\multicolumn{1}{c}{$\Delta_{\text{avg\_loss}}^{\text{epoch }5-4}$} \\
\midrule
Qwen2.5-3B   & 0.014957 & 0.013486 & 0.015686 & 0.014000 \\
Qwen2.5-7B   & 0.009729 & 0.022514 & 0.014172 & 0.001156 \\
Qwen2.5-14B  & 0.008515 & 0.017786 & 0.011157 & 0.005472 \\
Qwen2.5-32B  & 0.007143 & 0.018200 & 0.015557 & 0.003772 \\
Llama3.1-8B  & 0.015586 & 0.023344 & 0.005571 & 0.004915 \\
\bottomrule
\end{tabular}%
}
\caption{Comparison of per-epoch loss drops under \textit{positive-only} and \textit{negative-only} SFT on MMLU. Each entry reports $\Delta_{\text{pos}}-\Delta_{\text{neg}}$, where $\Delta$ is the average loss decrease within an epoch. Interpretation focuses on relative differences across epochs.}
\label{tab:avg-loss-delta}
\vspace{-12pt}
\end{table}

\subsection{Inference Perspective}
\label{sec:analysis_infer}

We examine how \textit{negative} supervision changes inference behavior. We use token-level policy entropy as a proxy for uncertainty and exploration during reasoning. Let $M_{\text{pos}}$ be the model fine-tuned on \textit{positives} from OpenMathReasoning, and $M_{\text{neg}}$ be the model fine-tuned on \textit{negatives}. To evaluate both in-domain and OOD behavior, we distill reference trajectories from Qwen3-8B on an in-domain math set (``Math'') and an OOD set (``Other''). We define the thinking span as tokens between \texttt{<think>} and \texttt{</think>}, and the answer span as tokens after \texttt{</think>}. We compute entropy under two protocols. \textbf{Off-policy} evaluates entropy along the teacher trajectory (teacher forcing). \textbf{On-policy} evaluates entropy along the model’s own generated trajectory under a fixed decoding rule. Entropy is computed from raw logits with $T{=}1$ and includes special boundary tokens.


\begin{figure}[t]
  \centering
  \includegraphics[width=\columnwidth]{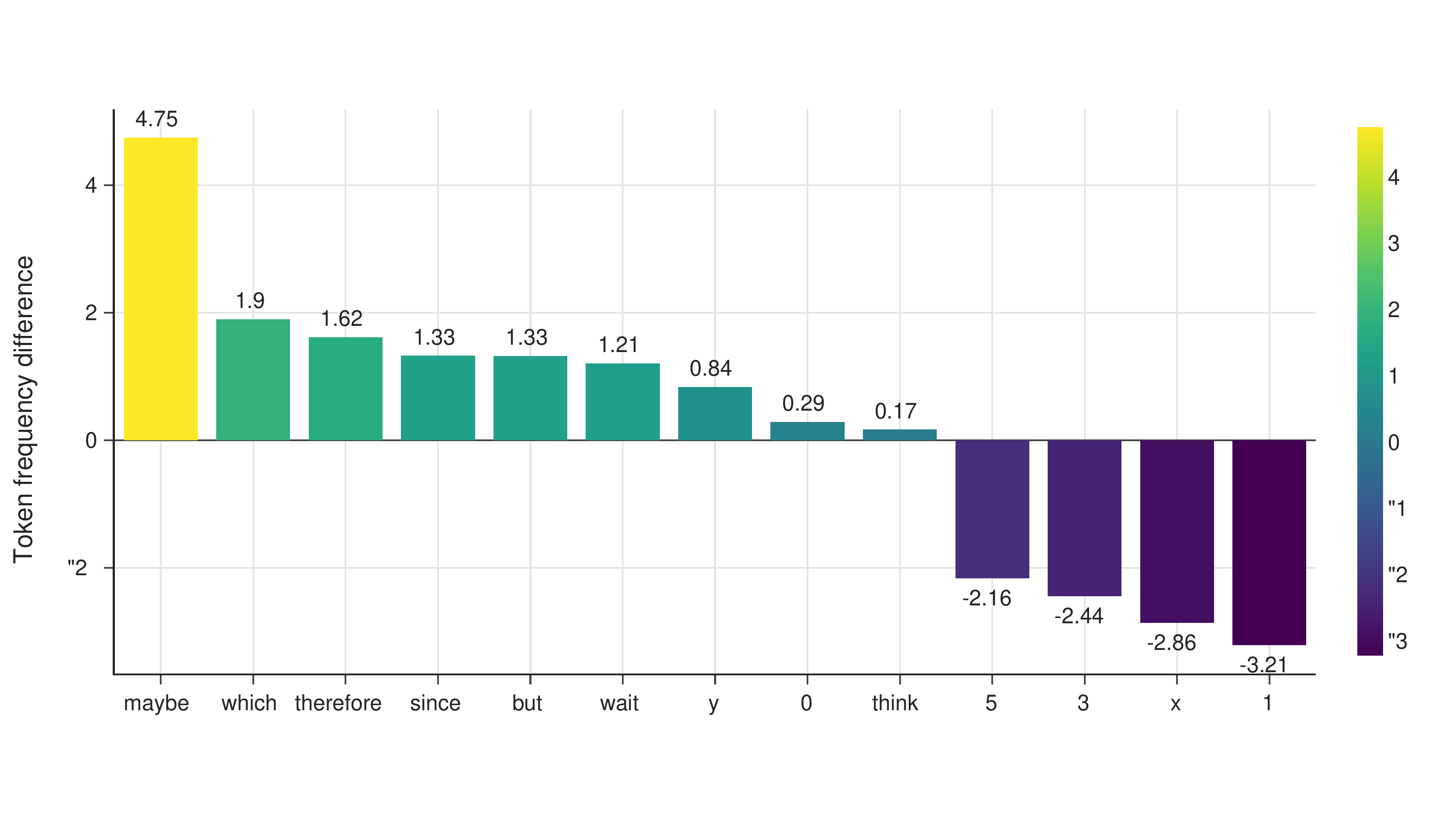}
  \caption{Token frequency differences between $M_{\text{neg}}$ and $M_{\text{pos}}$ on digits and high-entropy tokens.}
  \label{fig:freq_entropy}
\end{figure}

\begin{table}[t]
  \small
  \centering
  
  \resizebox{\columnwidth}{!}{
  
  \renewcommand{\arraystretch}{1.0}
  \begin{tabular}{cccccc}
\toprule
\textbf{Model} & \textbf{Setting} & \textbf{Data} &
$\bar H_{\text{think}}$ & $\bar H_{\text{ans}}$ & $\Delta H$ \\
\midrule
\multirow{4}{*}{$M_{\text{pos}}$}
  & \multirow{2}{*}{Off-policy} & Math  & 0.909 & 0.708 & 0.202 \\
  &                              & Other & 1.138 & 0.873 & 0.265 \\
\cmidrule(lr){2-6}
  & \multirow{2}{*}{On-policy}  & Math  & 0.753 & 0.601 & 0.153 \\
  &                              & Other & 0.669 & 0.757 & -0.088 \\
\midrule
\multirow{4}{*}{$M_{\text{neg}}$}
  & \multirow{2}{*}{Off-policy} & Math  & 1.212 & 0.883 & 0.329 \\
  &                              & Other & 1.427 & 0.992 & 0.435 \\
\cmidrule(lr){2-6}
  & \multirow{2}{*}{On-policy}  & Math  & 1.011 & 0.772 & 0.239 \\
  &                              & Other & 0.917 & 0.783 & 0.134 \\
\bottomrule
\end{tabular}
  }
    \caption{Policy entropy analysis on $M_{\text{pos}}$ and $M_{\text{neg}}$.}
  \label{tab:policy-entropy}
\vspace{-8pt}
\end{table}

Formally, let $\mathcal{V}$ be the vocabulary and $\theta$ the model parameters. The token-level entropy at step $t$ is
\begin{equation}
\label{eq:token-entropy}
\begin{aligned}
p_t(v) &\triangleq p_\theta\!\left(v \mid x, y_{<t}\right), \\
H_t(\theta \mid x, y_{<t})
&= - \sum_{v\in\mathcal V} p_t(v)\log p_t(v).
\end{aligned}
\end{equation}

where $p_\theta(\cdot \mid x, y_{<t})$ is the softmax distribution induced by the pre-softmax logits. For sample $i$, let $\mathcal{T}^{(i)}{\text{think}}$ and $\mathcal{T}^{(i)}{\text{ans}}$ denote token indices in the thinking and answer spans, determined by the teacher trajectory (off-policy) or the model trajectory (on-policy). We report mean span entropy:
\begin{equation}
\label{eq:span-entropy}
\begin{aligned}
\bar H_{\text{think}}^{(i)}
&= \frac{1}{\left|\mathcal T_{\text{think}}^{(i)}\right|}
\sum_{t \in \mathcal T_{\text{think}}^{(i)}} H_t, \\
\bar H_{\text{ans}}^{(i)}
&= \frac{1}{\left|\mathcal T_{\text{ans}}^{(i)}\right|}
\sum_{t \in \mathcal T_{\text{ans}}^{(i)}} H_t .
\end{aligned}
\end{equation}

and the boundary drop:
\begin{equation}
\Delta H^{(i)}=\bar H{\text{think}}^{(i)} - \bar H_{\text{ans}}^{(i)}.
\label{eq:delta-entropy}
\end{equation}
As presented in Table~\ref{tab:policy-entropy}, $M_{\text{neg}}$ maintains higher entropy throughout the thinking span and displays a sharper decline at the answer boundary. This dynamic reflects a strategy of broad exploration followed by decisive commitment, which correlates with its robust cross-domain transfer. Regarding baselines, off-policy entropy is inherently higher because teacher forcing exposes the model to contexts that can have low probability under its own policy. Crucially, however, the models exhibit contrasting behaviors under distribution shift. While $M_{\text{neg}}$ remains robust, $M_{\text{pos}}$ suffers a structural collapse on OOD data, where the entropy margin even reverses, indicating in-domain overfitting.

We further localize where uncertainty concentrates. Figure~\ref{fig:freq_entropy} compares high-entropy token usage in generated trajectories. Relative to $M_{\text{pos}}$, $M_{\text{neg}}$ produces more discourse and hesitation tokens (e.g., “maybe,” “wait,” “but”) and fewer numerals, indicating more budget allocated to connective exploration before committing to concrete computation. Figure~\ref{fig:case_study} illustrates the same effect qualitatively. During inference, $M_{\text{neg}}$ maintains more plausible continuations and explores more reasoning paths before settling on an answer.

\textbf{Overall, these results indicate that \textit{negative}-based supervision induces higher-entropy, more exploratory yet ultimately more decisive reasoning policies at inference time, which supports more robust cross-domain generalization.}

\begin{table*}[h]
\centering

\renewcommand{\arraystretch}{1.3}
\resizebox{\textwidth}{!}{
\begin{tabular}{ll|cccc|c|ccc|c|cc|c}
\hline
\multicolumn{2}{c|}{ }
 & \multicolumn{5}{c|}{Math Reasoning (In-Domain)}
 & \multicolumn{4}{c|}{General Reasoning (Out-of-Domain)}
 & \multicolumn{3}{c}{Other Reasoning (Out-of-Domain)} \\
\hline
Model & Setting
 & Math500 & Minerva & Olympia & AMC & \multicolumn{1}{c|}{Avg.}
 & MMLU & MMLU-Pro & BBH & \multicolumn{1}{c|}{Avg.}
 & ACPBench & HeadQA & Avg. \\
\hline

\multirow{2}{*}{Qwen2.5-3B}
 & Full
                & 60.80 & 26.10 & 23.26 & 35.00 & 36.29
                & 64.13 & \textbf{38.66} & 52.29 & 51.69 
                & 32.68 & 62.69 & 47.69 \\
\cline{2-14}
 & GLOW
                & \textbf{62.80} & \textbf{27.21} & \textbf{24.30} & \textbf{42.50} & \textbf{39.20} 
                & \textbf{64.49} & {38.63} & \textbf{53.20} & \textbf{52.11} 
                & \textbf{33.66} & \textbf{63.38} & \textbf{48.52} \\
\hline

\multirow{2}{*}{Qwen2.5-7B}
 & Full
                & 76.60 & 40.07 & 38.96 & 55.00 & 52.66 
                & 72.24 & 53.71 & 70.84 & 65.60 
                & 38.27 & 72.06 & 55.17 \\
\cline{2-14}
 & GLOW
                & \textbf{79.60} & \textbf{40.07} & \textbf{41.04} & \textbf{60.00} & \textbf{55.18} 
                & \textbf{73.99} & \textbf{55.77} & \textbf{71.99} & \textbf{67.25} 
                & \textbf{39.19} & \textbf{72.50} & \textbf{55.85} \\
\hline

\multirow{2}{*}{Qwen2.5-14B}
 & Full
                & 86.80 & 47.79 & 52.30 & 82.50 & 67.35 
                & 81.56 & 67.63 & 80.90 & 76.70 
                & 48.13 & 81.44 & 64.79 \\
\cline{2-14}
 & GLOW
                & \textbf{87.80} & \textbf{52.21} & \textbf{52.44} & \textbf{82.50} & \textbf{68.74} 
                & \textbf{82.53} & \textbf{68.70} & \textbf{81.65} & \textbf{77.63} 
                & \textbf{49.51} & \textbf{82.35} & \textbf{65.93} \\
\hline

\multirow{2}{*}{Qwen2.5-32B}
 & Full
                & 92.20 & 52.57 & 57.19 & 85.00 & 71.74 
                & 85.22 & 73.10 & 83.53 & 80.62
                & 50.67 & 84.90 & 67.79 \\
\cline{2-14}
 & GLOW
                & \textbf{93.40} & \textbf{54.41} & \textbf{59.11} & \textbf{92.50} & \textbf{74.86} 
                & \textbf{85.51} & \textbf{74.14} & \textbf{83.98} & \textbf{81.21} 
                & \textbf{51.97} & \textbf{85.19} & \textbf{68.58} \\
\hline

\multirow{2}{*}{Llama3.1-8B}
 & Full
                & 41.20 & 18.01 & 14.67 & 15.00 & 22.22
                & 62.48 & 36.88 & 55.12 & 51.49
                & 32.96 & 65.90 & 49.43 \\
\cline{2-14}
 & GLOW
                & \textbf{44.60} & \textbf{20.59} & \textbf{15.11} & \textbf{17.50} & \textbf{24.45} 
                & \textbf{63.80} & \textbf{38.34} & \textbf{58.17} & \textbf{53.44} 
                & \textbf{35.04} & \textbf{66.70} & \textbf{50.87} \\
\hline

\end{tabular}
}
\caption{Cross-domain performance of models trained on the \textbf{math reasoning} dataset. 
``Avg.'' denotes the average score within each group. \textbf{Bold} indicates the best results under the same model.}
\label{tab:exp_math}
\vspace{-3pt}
\end{table*}

\begin{table*}[h]

\renewcommand{\arraystretch}{1.3}
\resizebox{\textwidth}{!}{
\begin{tabular}{ll|cccc|c|ccc|c|cc|c}
\hline
\multicolumn{2}{c|}{ }
 & \multicolumn{5}{c|}{Math Reasoning (Out-of-Domain)} 
 & \multicolumn{4}{c|}{General Reasoning (In-Domain)} 
 & \multicolumn{3}{c}{Other Reasoning (Out-of-Domain)} \\
\hline
Model & Setting 
 & Math500 & Minerva & Olympia & AMC & \multicolumn{1}{c|}{Avg.} 
 & MMLU & MMLU-Pro & BBH & \multicolumn{1}{c|}{Avg.} 
 & ACPBench & HeadQA & Avg. \\
\hline
\multirow{2}{*}{Qwen2.5-3B} 
 & Full        & 58.20 & 23.16 & 25.19 & 35.00 & 35.39 
                & 66.74 & 40.82 & \textbf{53.35} & 53.64 
                & 35.70 & 67.61 & 51.66 \\
\cline{2-14}
 & GLOW 
                & \textbf{61.40} & \textbf{29.41} & \textbf{25.78} & \textbf{40.00} & \textbf{39.15} 
                & \textbf{67.09} & \textbf{41.27} & {52.61} & \textbf{53.66} 
                & \textbf{36.20} & \textbf{69.15} & \textbf{52.68} \\
\hline

\multirow{2}{*}{Qwen2.5-7B} 
 & Full        & 75.60 & 38.60 & 40.15 & 47.50 & 50.46 
                & 73.14 & \textbf{51.15} & 71.30 & 65.20 
                & 42.18 & 72.76 & 57.47 \\
\cline{2-14}
 & GLOW 
        & \textbf{78.20} & \textbf{41.18} & \textbf{43.70} & \textbf{60.00} & \textbf{55.77} 
        & \textbf{74.51} & {51.13} & \textbf{71.99} & \textbf{65.88} 
        & \textbf{43.56} & \textbf{75.35} & \textbf{59.46} \\
\hline

\multirow{2}{*}{Qwen2.5-14B} 
 & Full        & 82.20 & 43.01 & 51.85 & 70.00 & 61.77 
                & 78.13 & 59.57 & 80.56 & 72.75 
                & 48.87 & 79.94 & 64.41 \\
\cline{2-14}
 & GLOW 
        & \textbf{85.00} & \textbf{48.09} & \textbf{54.22} & \textbf{70.00} & \textbf{64.33} 
        & \textbf{79.97} & \textbf{62.78} & \textbf{82.32} & \textbf{75.02} 
        & \textbf{50.95} & \textbf{82.20} & \textbf{66.58} \\
\hline

\multirow{2}{*}{Qwen2.5-32B} 
 & Full        & 86.60 & 46.69 & 55.70 & 80.00 & 67.25 
                & 79.06 & 61.15 & 79.94 & 73.38
                & 49.89 & 83.01 & 66.45 \\
\cline{2-14}
 & GLOW
                & \textbf{89.00} & \textbf{47.06} & \textbf{58.67} & \textbf{82.50} & \textbf{69.31} 
                & \textbf{80.81} & \textbf{64.72} & \textbf{81.98} & \textbf{75.84} 
                & \textbf{52.08} & \textbf{83.73} & \textbf{67.91} \\
\hline

\multirow{2}{*}{Llama3.1-8B} 
 & Full        & 20.00 & 15.81 & 6.52 & 2.50 & 11.21 
                & 66.49 & 40.56 & 53.73 & 53.59
                & 36.06 & 69.55 & 52.81 \\
\cline{2-14}
 & GLOW
                & \textbf{24.80} & \textbf{20.59} & \textbf{6.96} & \textbf{12.50} & \textbf{16.21} 
                & \textbf{68.52} & \textbf{42.96} & \textbf{57.53} & \textbf{56.33} 
                & \textbf{39.72} & \textbf{72.57} & \textbf{56.15} \\
\hline

\end{tabular}
}
\centering
\caption{Cross-domain performance of models trained on the \textbf{general reasoning} dataset. 
``Avg.'' denotes the average score within each group. \textbf{Bold} indicates the best results under the same model.}
\label{tab:exp_mmlu}
\vspace{-12pt}
\end{table*}

\section{From Negatives to Effective Full-Data Training} 
In this section, we move beyond the empirical finding that \textit{negatives} improve OOD generalization. Training on \textit{negatives} alone remains a rejection-based strategy and still fails to use supervision efficiently. Our goal is to improve both in-domain and OOD performance while using data more effectively. We therefore target the training objective and propose a simple mechanism that adapts sample weights based on learning progress.

\subsection{GLOW: Gain-Based Loss Weighting}

Our analysis suggests that \textit{negatives} help by injecting optimization diversity, which broadens the learned reasoning space. 
This motivates reweighting SFT toward undercovered patterns. 
\ours{} quantifies each sample’s gain by its inter-epoch loss reduction. 
A small gain indicates limited effective coverage under the current trajectory. 
\ours{} then upweights such samples via an adaptive scaling function, steering updates toward complementary directions and improving generalization.

Let $\ell_i^{(t)}$ denote the loss of sample $i$ at epoch $t$. We quantify a sample’s learning progress as its inter-epoch loss reduction:
$\Delta_i^{(t)}=\ell_i^{(t-1)}-\ell_i^{(t)}$.
A small $\Delta_i^{(t)}$ indicates that the sample remains insufficiently learned and may encode underrepresented patterns, whereas a large $\Delta_i^{(t)}$ suggests diminishing marginal utility. We therefore upweight small-$\Delta$ samples via
\begin{equation}
w_i^{(t)} = \alpha \bigl(1 - \sigma\left( \beta \Delta_i^{(t)} \right)\bigr),
\label{eq:weight}
\end{equation}
where $\sigma(\cdot)$ is the sigmoid function and $\alpha,\beta$ are scaling hyperparameters. For the first epoch, we set $w_i^{(1)}=1$ for all samples. The reweighted objective is
\begin{equation}
\mathcal{L}^{(t)}_{\text{\ours{}}}(\theta)
=\sum_{i=1}^{N} w_i^{(t)}\,\ell_i(\theta).
\label{eq:resft}
\end{equation}

\paragraph{Why it works.}
The inter-epoch loss reduction $\Delta_i^{(t)}$ measures how much sample $i$ benefits from recent updates. Under standard $L$-smoothness, for an update $\theta'=\theta-\eta G^{(t)}$ with $G^{(t)}=\frac{1}{N}\sum_j w_j^{(t)}\nabla \ell_j(\theta)$, a first-order expansion gives
\[
\Delta_i^{(t)} \approx \ell_i(\theta)-\ell_i(\theta') \approx \eta \big\langle \nabla \ell_i(\theta),\, G^{(t)} \big\rangle,
\]
so $\Delta_i^{(t)}$ is closely tied to the alignment between the current update direction and the sample gradient. Thus, small $\Delta_i^{(t)}$ indicates patterns weakly covered by optimization, whereas large $\Delta_i^{(t)}$ suggests diminishing marginal utility. Since $w_i^{(t)}=\alpha\!\left(1-\sigma(\beta\Delta_i^{(t)})\right)$ is decreasing in $\Delta_i^{(t)}$, \ours{} prioritizes small-$\Delta$ samples to steer training toward complementary directions, increasing gradient diversity and exploration to improve generalization. See Appendix~\ref{app:proof} for detailed derivation.


\subsection{Discussion with Prior Works}
Parallel to our focus on negative trajectories, recent reasoning-oriented RL approaches also leverage negative signals, yet primarily to penalize undesired behaviors via reward structuring and credit assignment~\citep{zhu2025negative-reinforcement,liu2025erpo,yang2025unearthing,nan2025ngrpo}. In contrast, GLOW investigates \textit{negatives} within the SFT stage and establishes them as a source of direct supervision: rather than merely being suppressed, \textit{negatives} are utilized to broaden the reasoning space, thereby enhancing OOD generalization.

Regarding objective design, prior SFT reweighting typically targets optimization imbalance by utilizing current loss to down-weight easy samples, a process that is effectively memoryless~\cite{lin2017focal,bengio2009curriculum}. In contrast, GLOW targets coverage: it upweights samples exhibiting stagnant progress, directing optimization toward underexplored reasoning patterns.

\subsection{Experimental Results}
Building on the theoretical analysis, we empirically validate the effectiveness of~\ours{} in the SFT stage. All other experimental settings are the same as~\ref{sec:exp_setting} and details are described in Appendix~\ref{app:dataset}. 

\paragraph{\ours{} improves cross-domain generalization without sample filtering.}
We apply \ours{} to a randomly shuffled mixture of \textit{positives} and \textit{negatives} and observe consistent gains across domains and model scales. For brevity, we report only standard SFT on the mixed data and \ours{}. Results for \textit{positive-only} and \textit{negative-only} SFT are provided in Tables~\ref{tab:full_math} and~\ref{tab:full_mmlu}.
Table~\ref{tab:exp_math} demonstrates that \ours{} consistently enhances in-domain performance across all model scales and achieves superior OOD results. Specifically, for Qwen2.5-7B, \ours{} reaches 55.18 in-domain and 67.25 OOD while maintaining competitive general reasoning abilities. Similar gains are observed in models trained on general reasoning data. Further results in Table~\ref{tab:exp_mmlu} show that on Qwen2.5-14B, \ours{} boosts OOD math and reasoning performance by 2.56 and 2.17 points, respectively. Overall, \ours{} maximizes data utilization by learning from all trajectories, yielding robust gains across diverse benchmarks and settings. We also do the ablation of hyperparameters in Appendix~\ref{app:hyper}.


\paragraph{\ours{} tends to upweight \textit{negatives}.}
We further examine the samples prioritized by \ours{}. Appendix~\ref{app:neg_ratio} reveals that \ours{} prioritizes \textit{negatives} during early training stages, suggesting that gain-based weights primarily target harder and under-represented reasoning patterns.

\begin{table}[t]
\centering
\small
\setlength{\tabcolsep}{8pt}
\renewcommand{\arraystretch}{0.8}
\resizebox{1.0\linewidth}{!}{
\small
\begin{tabular}{l c l c c c}
\toprule
\textbf{Setting} & \textbf{Train} & \textbf{Test} &
$\bar H_{\text{think}}$ & $\bar H_{\text{ans}}$ & $\Delta H$ \\
\midrule
\multirow{4}{*}{Full}
  & \multirow{2}{*}{Math} & Math  & 0.36 & 0.22 & 0.14 \\
  &                        & Other & 1.24 & 1.38 & -0.14 \\
\cmidrule(lr){2-6}
  & \multirow{2}{*}{MMLU} & Math  & 0.54 & 0.34 & 0.20 \\
  &                        & Other & 0.96 & 0.98 & -0.02 \\
\midrule
\multirow{4}{*}{GLOW}
  & \multirow{2}{*}{Math} & Math  & 0.71 & 0.35 & 0.36 \\
  &                        & Other & 1.52 & 1.30 & 0.22 \\
\cmidrule(lr){2-6}
  & \multirow{2}{*}{MMLU} & Math  & 0.89 & 0.52 & 0.37 \\
  &                        & Other & 1.44 & 1.21 & 0.23 \\
\bottomrule
\end{tabular}
}
\caption{Policy entropy changes with and without GLOW under various settings.}
\label{tab:glow_entropy}
\vspace{-5pt}
\end{table}

\paragraph{\ours{} encourages exploration while preserving answer commitment.}
Table~\ref{tab:glow_entropy} shows \ours{} consistently increases thinking span entropy across settings (e.g., 0.36 to 0.71 from Math to Math, and 0.96 to 1.44 from MMLU to Other domains), while answer span entropy remains stable or decreases under OOD.
This suggests \ours{} encourages broader exploration during reasoning while keeping answers relatively decisive, consistent with its generalization gains.



\paragraph{\ours{} serves as a superior initialization for subsequent RL training.}
\begin{table}[t]
\centering
\resizebox{\columnwidth}{!}{%
\begin{tabular}{lccccc}
\toprule
\textbf{Setting} &   \textbf{Math500} & \textbf{Minerva} &   \textbf{AMC} & \textbf{MMLU} & \textbf{MMLU-Pro} \\
\midrule
SFT &  76.60 & 40.07 &  55.00 & 72.24 & 53.71 \\
\ours{}   &  79.60 & 40.07  & \textbf{60.00} & 73.99 & 55.77 \\
\midrule
SFT + RL  & 78.20 & 40.24 & 52.50 & 72.82 & 53.95 \\
\ours{} + RL & \textbf{80.20} & \textbf{42.28} &  57.50 & \textbf{76.47} & \textbf{57.37} \\
\bottomrule
\end{tabular}%
}
\caption{Controlled comparison of post-training on GSM8K for Qwen2.5-7B-base. \ours{} improves OOD metrics both before and after RL, indicating a stronger initialization for RL post-training.}
\label{tab:rl_init}
\vspace{-13pt}
\end{table}
Starting from Qwen2.5-7B-base, we train on GSM8K with four settings: (i) standard SFT, (ii) standard SFT followed by RL post-training, (iii) SFT with \ours{}, and (iv) \ours{}-based SFT followed by the same RL post-training. 
Using GRPO~\cite{shao2024deepseekmathpushinglimitsmathematical} with fixed RL data, optimizer, and hyperparameters, we vary only the SFT objective to isolate its impact. Table~\ref{tab:rl_init} shows that \ours{} improves OOD performance before RL and remains superior after RL, outperforming the RL model initialized from standard SFT. 
This indicates \ours{} yields stronger SFT initialization that transfers to RL post-training.


\section{Conclusion}

We show that negative reasoning trajectories can improve SFT generalization and mitigate OOD degradation. Through data, training, and inference analyses, we identify why \textit{negatives} help and how they shape optimization and model behavior. Building on these findings, we propose Gain-based LOss Weighting (\ours{}), which upweights undercovered examples using inter-epoch loss reduction, yielding more data-efficient training and consistent cross-domain gains across diverse benchmarks.


\section*{Limitations}
Our study primarily examines gain-based reweighting in the supervised fine-tuning stage of reasoning post-training, and we leave its interaction with subsequent RLHF or other reinforcement learning stages as an exciting direction for future work. In addition, our experiments focus on text-only chain-of-thought data for math and multi-task knowledge benchmarks with a small set of open-source backbones, so a natural next step is to extend the same analysis and method to broader task families, larger model scales and multimodal or tool-augmented settings, building on the phenomena and gains established in this work.


\bibliography{custom}

\clearpage
\appendix

\section{Appendix}
\subsection{Experiments Setup}\label{app:dataset}
\paragraph{Distillation data curation}
We conduct experiments on mathematical reasoning and common sense, using Qwen3-8B~\citep{yang2025qwen3} to distill reasoning trajectories. For mathematics, we collect data from OpenMathReasoning~\citep{moshkov2025aimo2}, and for common sense from MMLU~\citep{hendryckstest2021,hendrycks2021ethics}. Each trajectory is labeled as \textit{positive} if the final answer matches the ground truth and \textit{negative} otherwise. To ensure that all samples preserve complete reasoning structures and differ only in correctness, we discard instances exceeding 8,192 tokens. We then sample \textit{positive} and \textit{negative} data in a 1:1 ratio, resulting in 7.2k instances for mathematics and 17.4k for common sense.

\paragraph{Training Details}
We conduct experiments on the Qwen2.5 series (3B, 7B, 14B, 32B)~\citep{qwen2.5} and LLaMA-3.1-8B\citep{dubey2024llama}. All models are fine-tuned for 20 epochs with a batch size of 128, using a cosine learning rate scheduler with 10\% warm-up steps and a maximum learning rate of $5 \times 10^{-5}$. We set the training length to 20 epochs, as the loss does not converge earlier and benchmark performance continues to improve up to this point.

\paragraph{Evaluation Details}
Following~\citet{huan2025does,yuan2025incentivizing}, we evaluate models on three categories of benchmarks: (1)~\textbf{mathematical reasoning}: MATH500~\citep{hendrycks2024measuring}, OlympiaBench~\citep{he2024olympiadbench}, MinervaMath~\citep{lewkowycz2022solving}, and the competition-level AMC2023~\citep{amc23dataset}; (2)~\textbf{common sense reasoning}: MMLU, MMLU-Pro~\citep{wang2024mmlu}, and BBH~\citep{suzgun2022challenging}; (3)~\textbf{other OOD reasoning}: ACPBench~\citep{kokel2025acpbench} for planning, and HeadQA~\citep{vilares2019head} for medicine. Model performance is measured by accuracy. Evaluation uses the codebase from \citep{yuan2025incentivizing}, with sampling temperature 0.6, top-p 0.95, one sample per input, and max generation length 32,768 tokens.

We define in-domain and out-of-domain (OOD) evaluation based on the training data distribution. For models fine-tuned on mathematical reasoning tasks, in-domain evaluation uses mathematical problems while OOD evaluation employs other task categories. Conversely, models trained on MMLU are evaluated in-domain on commonsense tasks and OOD on the remaining domains. We compare three training strategies: using only \textit{positive} samples, only \textit{negative} samples, and a balanced combination of both.

Artifact Licenses and Intended Use: The models, the evaluation benchmarks and datasets are public artifacts. We utilize them in strict accordance with their respective licenses. Our use of these artifacts for SFT and reasoning evaluation is consistent with their intended use for scientific research.

\subsection{Detailed Theoretical Derivation}
\label{app:proof}

We provide a theoretical framework to motivate the dynamic reweighting mechanism in Eq.~\ref{eq:weight}. Under idealized smoothness and stability assumptions, our derivation suggests that GLOW can improve optimization conditioning. The core intuition is that a sample’s short-horizon loss reduction acts as a proxy for the alignment between its gradient and the current update direction. Consequently, assuming low-gain samples align with undercovered subspaces, upweighting them adds positive semidefinite curvature along complementary directions. This potentially increases the spectrum of the weighted Fisher proxy, improves local conditioning, and reduces algorithmic sensitivity.
\subsubsection{Setup and Notation}
We analyze a single update step. Let $\theta$ be the current parameters and
$\ell_i(\theta)$ the per-sample loss. Define
\[
g_i \triangleq \nabla_\theta \ell_i(\theta),
\qquad
G \triangleq \frac{1}{N}\sum_{i=1}^N w_i\, g_i ,
\]
where $w_i \ge 0$ are the weights used in the current step.
The update is
\[
\theta^{'} = \theta - \eta\, G.
\]
We also define the weighted surrogate objective
\[
R_w(\theta) \triangleq \frac{1}{N}\sum_{i=1}^N w_i\,\ell_i(\theta),
\]
where the weights $\{w_i\}$ are held fixed during this step.

The weighted empirical Fisher proxy at $\theta$ is
\[
F_w(\theta) \triangleq \frac{1}{N}\sum_{i=1}^N w_i\, g_i g_i^\top,
\]
and we abbreviate $F_w(\theta)$ to $F_w$ when the dependence is clear.



\subsubsection{Notation and Standing Assumptions}
\begin{assumptionenv}[Smoothness, boundedness, and curvature injection]
\leavevmode
\begin{enumerate}[label=(A\arabic*)]
  \item Each $\ell_i(\theta)$ is twice differentiable and $L$-smooth, namely
        $\| \nabla^2_\theta \ell_i(\theta)\|_{\mathrm{op}} \le L$ for all $i$ and $\theta$.
  \item Gradients are uniformly bounded: $\|g_i(\theta)\|_2 \le G_{\max}$.
  \item The learning rate $\eta$ is small enough so that second-order remainders
        in Taylor expansions are controlled by $L$.
  \item (Fisher Hessian closeness for the surrogate) At the iterates where the analysis
        is applied, the Hessian of $R_w$ satisfies
        \[
        \bigl\|\nabla^2_\theta R_w(\theta) - F_w(\theta)\bigr\|_{\mathrm{op}} \le \delta.
        \]
  \item (Coverage on low-curvature directions) Let $U$ be a $k$-dimensional
        subspace with projector $P_U$ that captures low-curvature directions of $F_w$
        (e.g., the span of the $k$ smallest-eigenvalue eigenvectors of $F_w$).
        Intuitively, the reweighting rule upweights samples with small gain, whose gradients
        are weakly aligned with the current update direction and thus tend to contribute
        complementary curvature along undercovered directions.
        Suppose the rule increases weights on a set $T$ by increments $\delta w_i \ge 0$,
        inducing
        \[
        \Delta F \triangleq \frac{1}{N}\sum_{i\in T} \delta w_i\, g_i g_i^\top .
        \]
        We assume this update provides nontrivial coverage on $U$:
        \[
        P_U \Delta F P_U \succeq \frac{\gamma}{k}\,P_U .
        \]

\end{enumerate}
\end{assumptionenv}

\subsubsection{From Gain to Gradient Alignment}
\begin{lemma}[Loss reduction and gradient alignment]
\label{lemma:delta}
Under (A1)--(A3), after the update $\theta^{'}=\theta-\eta G$, we have
\[
\Delta_i
\triangleq
\ell_i(\theta) - \ell_i(\theta^{'})
=
\eta\, g_i^\top G - \frac{1}{2}\eta^2\, G^\top H_i(\xi_i)\, G
\]
for some $\xi_i$ on the line segment between $\theta$ and $\theta^{'}$, where
$H_i(\xi_i)=\nabla^2_\theta \ell_i(\xi_i)$. Moreover,
\[
\Big|\Delta_i - \eta\, g_i^\top G\Big|
\le
\frac{1}{2}L\eta^2 \|G\|_2^2.
\]
\end{lemma}

\begin{proof}
A second-order Taylor expansion of $\ell_i(\theta-\eta G)$ around $\theta$
gives the stated expression, and $L$-smoothness bounds the remainder.
\end{proof}

Lemma~\ref{lemma:delta} implies that, up to a controlled second-order term,
the gain $\Delta_i$ is large when $g_i$ aligns with the update direction $G$,
and small when $g_i$ is weakly aligned. Therefore, using small inter-epoch gain
as a signal to upweight samples is consistent with prioritizing directions that
are undercovered by recent optimization.


\subsubsection{PSD Augmentation of the Fisher Proxy}

\begin{lemma}[Positive weight increments induce PSD augmentation]
\label{lemma:deltaF}
If weights change by increments $\delta w_i \ge 0$ for $i\in T$, then the induced
change in the weighted Fisher is
\[
\Delta F = \frac{1}{N}\sum_{i\in T} \delta w_i\, g_i g_i^\top,
\]
which is positive semidefinite. Consequently, the updated Fisher
$F_w' = F_w + \Delta F$ satisfies $F_w' \succeq F_w$, and its eigenvalues
are monotonically nondecreasing.
\end{lemma}

\begin{proof}
Each $g_i g_i^\top$ is symmetric and positive semidefinite. With $\delta w_i\ge 0$,
every term $\delta w_i g_i g_i^\top$ is positive semidefinite, hence so is their average
$\Delta F$. Thus $F_w' = F_w + \Delta F \succeq F_w$, and eigenvalue monotonicity follows
from standard Weyl-type inequalities.
\end{proof}
\subsubsection{Improving Low-Curvature Directions}

Lemma~\ref{lemma:deltaF} guarantees a PSD augmentation $F_w' = F_w + \Delta F$,
but PSD alone does not ensure improved curvature along the bottleneck directions:
$\Delta F$ could concentrate on already well-conditioned directions and leave the
low-curvature subspace unchanged. Assumption (A5) rules out this degeneracy by
requiring the reweighting update to provide nontrivial coverage on $U$, which is
consistent with upweighting small-gain samples whose gradients are weakly aligned
with the current update and tend to contribute complementary directions.

\begin{lemma}[Improvement on a $k$-dimensional subspace]
\label{lemma:rank_gain}
Let $U$ be a $k$-dimensional subspace with projector $P_U$. Under (A5),
\[
\lambda_{\min}(F_w'|_U)
\ge
\lambda_{\min}(F_w|_U) + \frac{\gamma}{k},
\]
where $F_w|_U$ denotes the restriction of $F_w$ to $U$.
\end{lemma}

\begin{proof}
By (A5), $\Delta F|_U \succeq (\gamma/k)\,P_U$, so
$\lambda_{\min}(\Delta F|_U)\ge \gamma/k$.
Since $F_w'|_U = F_w|_U + \Delta F|_U$ and both are symmetric,
\begin{equation*}
\begin{aligned}
\lambda_{\min}(F_w'|_U)
&\ge \lambda_{\min}(F_w|_U) + \lambda_{\min}(\Delta F|_U) \\
&\ge \lambda_{\min}(F_w|_U) + \frac{\gamma}{k}.
\end{aligned}
\end{equation*}
\end{proof}

The same coverage condition (A5) also implies an average-curvature increase on $U$:
\begin{equation*}
\begin{aligned}
\frac{1}{k}\,\mathrm{tr}(P_U F_w')
&=
\frac{1}{k}\,\mathrm{tr}(P_U F_w)
+
\frac{1}{k}\,\mathrm{tr}(P_U \Delta F) \\
&\ge
\frac{1}{k}\,\mathrm{tr}(P_U F_w) + \frac{\gamma}{k}.
\end{aligned}
\end{equation*}





\subsubsection{Transferring Improvement from Fisher to Hessian}

\begin{lemma}[Fisher Hessian transfer on $U$]
\label{lemma:fisher_hessian}
Let $H(\theta)=\nabla^2_\theta R_w(\theta)$ and $H'(\theta)=\nabla^2_\theta R_{w'}(\theta)$
be the Hessians of the surrogate objectives associated with $F_w$ and $F_w'$,
respectively. Under (A4),
\[
\lambda_{\min}(H'|_U)
\ge
\lambda_{\min}(H|_U) + \frac{\gamma}{k} - 2\delta.
\]
\end{lemma}

\begin{proof}
For any symmetric matrices $A,B$, $|\lambda_{\min}(A)-\lambda_{\min}(B)|\le \|A-B\|_{\mathrm{op}}$.
Applying this to $(H,F_w)$ and $(H',F_w')$ and using (A4) yields
\begin{equation*}
\begin{aligned}
\lambda_{\min}(H'|_U)
&\ge \lambda_{\min}(F_w'|_U) - \delta \\
&\ge \lambda_{\min}(F_w|_U) + \frac{\gamma}{k} - \delta \\
&\ge \lambda_{\min}(H|_U) + \frac{\gamma}{k} - 2\delta .
\end{aligned}
\end{equation*}

where the middle inequality uses Lemma~\ref{lemma:rank_gain}.
\end{proof}
\subsubsection{Conditioning and Stability-Based Generalization}

\begin{lemma}[Improved conditioning reduces parameter sensitivity]
\label{lemma:stability}
Assume a restricted strong convexity condition on $U$: $\lambda_{\min}(H|_U)\ge \mu$,
and standard Lipschitz conditions for gradients hold.
Then the algorithmic stability scale is inversely proportional to $\mu$.
Consequently, increasing $\lambda_{\min}(H|_U)$ to $\mu'=\mu+\gamma/k-2\delta$
reduces sensitivity to data perturbations and yields a smaller stability-based
generalization bound; see \citet{bousquet2002stability,hardt2016train}.
\end{lemma}

\begin{proposition}[Conditioning and generalization improvement]
\label{prop:main}
Under (A1)--(A5), the reweighting rule induces a PSD Fisher augmentation and improves
curvature on the low-curvature subspace $U$. In particular:
\begin{enumerate}
  \item (Curvature on $U$) The Fisher proxy satisfies
        \begin{equation*}
        \begin{aligned}
        \frac{1}{k}\operatorname{tr}(P_U F_w' P_U)
        &\ge
        \frac{1}{k}\operatorname{tr}(P_U F_w P_U) + \frac{\gamma}{k}, \\
        \lambda_{\min}(F_w'|_U)
        &\ge
        \lambda_{\min}(F_w|_U) + \frac{\gamma}{k}.
        \end{aligned}
        \end{equation*}
  \item (Hessian transfer) The Hessian of the surrogate objective satisfies
        \[
        \lambda_{\min}(H'|_U)
        \ge
        \lambda_{\min}(H|_U) + \frac{\gamma}{k} - 2\delta.
        \]
  \item (Stability and generalization) If $\lambda_{\min}(H|_U)\ge\mu$, then after reweighting,
        the effective curvature lower bound increases to
        $\mu'=\mu+\gamma/k-2\delta$, which improves stability-based generalization bounds.
\end{enumerate}
\end{proposition}

\begin{proof}
Item 1 follows from Lemma~\ref{lemma:deltaF}, assumption (A5), and Lemma~\ref{lemma:rank_gain}.
Item 2 follows from Lemma~\ref{lemma:fisher_hessian}. Item 3 follows from Lemma~\ref{lemma:stability}.
\end{proof}

\noindent
In summary, gain-based reweighting uses small gain as a signal of weak alignment with recent updates, upweights such samples, injects curvature into undercovered directions through PSD Fisher augmentation, and improves local conditioning on low-curvature subspaces, which supports stronger stability-based generalization guarantees.
\begin{figure}[ht]
    \centering
    \begin{subfigure}{0.455\textwidth}
        \centering
        \includegraphics[width=\linewidth]{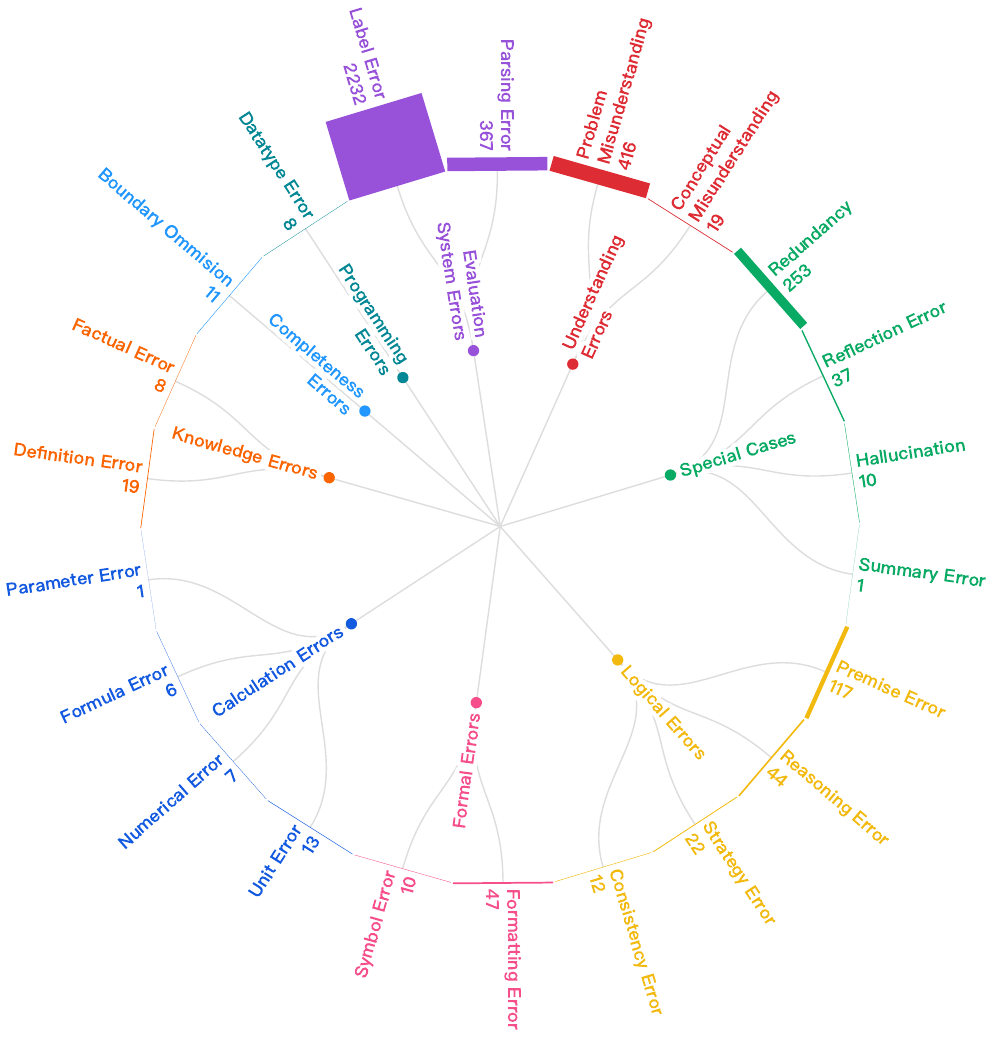}
        \caption{Error distribution in OpenMathReasoning.}
        \label{fig:openmath_errors}
    \end{subfigure}
    \begin{subfigure}{0.495\textwidth}
        \centering
        \includegraphics[width=\linewidth]{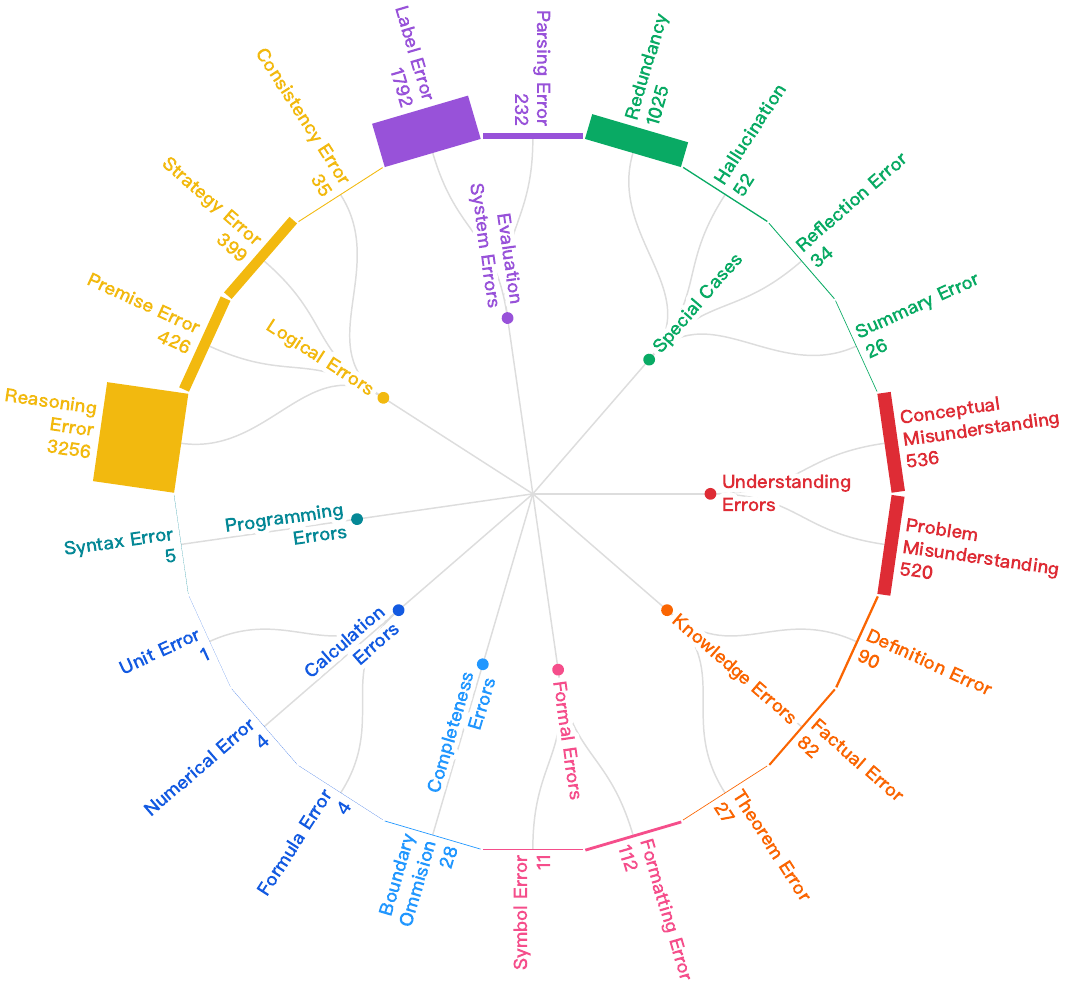}
        \caption{Error distribution in MMLU.}
        \label{fig:mmlu_errors}
    \end{subfigure}
    \caption{Detailed categorization of \textit{negative} samples in OpenMathReasoning and MMLU.}
    \label{fig:error_distribution}
\end{figure}

\begin{figure}[htbp!]
    \centering
        \small
        \includegraphics[width=\linewidth]{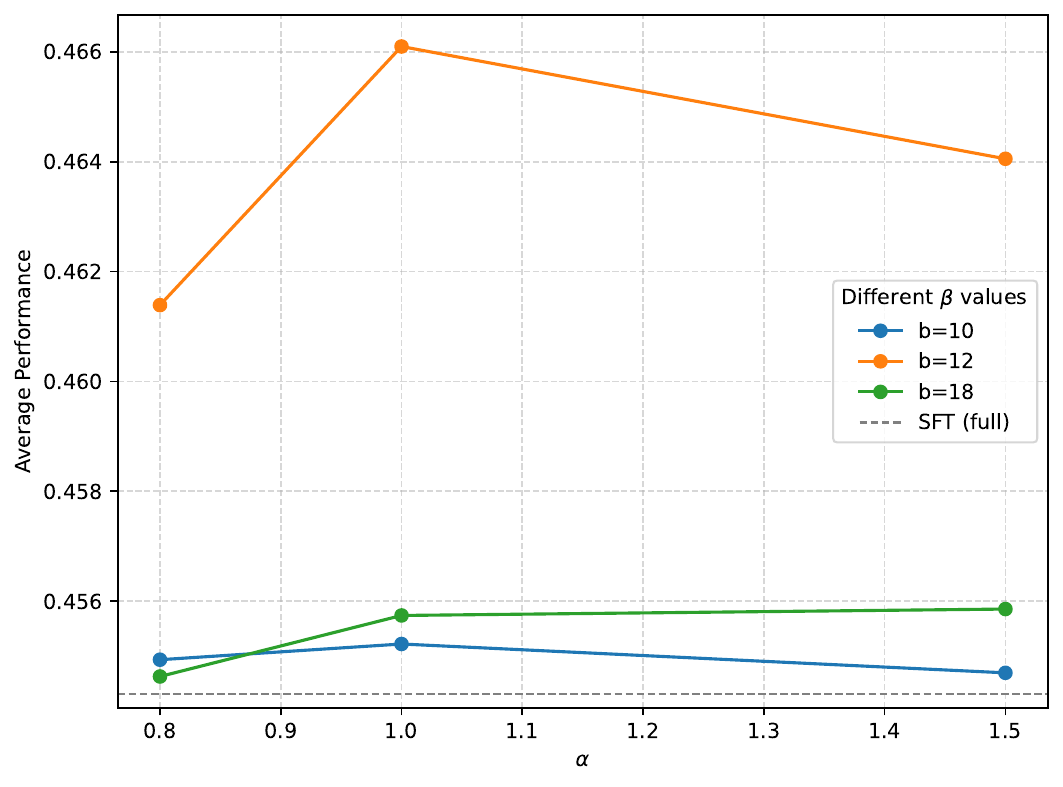}
        \captionof{figure}{Ablation study on the hyperparameters $\alpha$ and $\beta$. GLOW exhibits stable performance across different settings, demonstrating the robustness of the reweighting formulation.} 
        \label{fig:ablation}
\end{figure}

\subsection{IRM View of Diverse Negative Trajectories}
\label{app:irm_view}

This section formalizes our interpretation through Invariant Risk Minimization (IRM)~\citep{arjovsky2019invariant} in an autoregressive language modeling setting. Let $\mathcal{E}$ denote the set of environments induced by error categories of \textit{negative} trajectories. Each environment $e\in\mathcal{E}$ corresponds to a distribution $D^{e}$ over sequences $(x,y)$, where $x$ is the input and $y$ is the target reasoning trajectory.

We decompose the language model into a shared representation map $\Phi$ and a shared next-token predictor $w$, where $\Phi$ denotes the model body and $w$ denotes the vocabulary projection head. IRM seeks a representation $\Phi$ and a predictor $w$ such that the same $w$ is optimal across all environments when paired with $\Phi$:
\begin{equation}
\label{eq:irm_llm_obj}
\begin{cases}
\displaystyle \min_{\Phi,\,w}\ \sum_{e\in\mathcal{E}} R^{e}(w\circ \Phi), \\
\text{s.t. }\ 
w \in \arg\min_{w'} R^{e}(w'\circ \Phi),\ \forall e\in\mathcal{E}.
\end{cases}
\end{equation}

The per-environment autoregressive risk is
\begin{equation}
\label{eq:irm_llm_risk}
\begin{aligned}
R^{e}(w\circ \Phi)
&=\mathbb{E}_{(x,y)\sim D^{e}}
\Bigg[
\sum_{t=1}^{|y|}
\ell\!\left(
w\!\left(\Phi(x, y_{<t})\right),\, y_t
\right)
\Bigg],
\end{aligned}
\end{equation}
where $\ell$ denotes the cross-entropy loss.

Because $w$ is shared across all $e\in\mathcal{E}$, the shared-optimality constraint encourages $\Phi$ to encode reasoning features that remain predictive across heterogeneous error environments. Under our interpretation, \textit{negative} trajectories enlarge $\mathcal{E}$ by covering many error categories, which explains why diversity in \textit{negatives}, rather than any single error type, can improve robustness and OOD generalization.

\subsection{Detailed Taxonomy of Negative Training Samples}
\label{app:cat}
We provide statistics on the detailed categorization of \textit{negative} samples in our training dataset. As shown in Figure~\ref{fig:openmath_errors} and Figure~\ref{fig:mmlu_errors}, the error types of samples from OpenMathReasoning and MMLU that are not selected by reject sampling can be grouped into nine major categories and twenty-four subcategories. Although the distribution across categories is imbalanced, the errors still exhibit a broad coverage, ensuring a comprehensive representation of error types.

\subsection{Hyperparameter Sensitivity of~\ours{}}
\label{app:hyper}
As shown in Figure~\ref{fig:ablation}, ~\ours{} yields modest improvements over the full-SFT reference in most configurations.  Varying $\alpha$ between 0.8 and 1.5 leads to small changes, and $\beta=12$ is generally stronger than $\beta=10$ or $\beta=18$ at matched $\alpha$. These results suggest incremental gains with moderate hyperparameter choices in our setup.
\subsection{Training Loss on OpenMathReasoning and MMLU}
\label{app:loss}
Figure~\ref{fig:all} compares training losses for all models on OpenMathReasoning and MMLU under \textit{positive} and \textit{negative} settings.
\subsection{Model Performance Evolution Across Epochs}
\label{app:progress}
\begin{table*}[t]
\centering

\scriptsize
\setlength{\tabcolsep}{3.2pt}
\renewcommand{\arraystretch}{1.12}

\begin{subtable}[t]{0.485\textwidth}
\centering
\caption{Qwen2.5-7B is fine-tuned on the \textbf{math reasoning} dataset using \textbf{\textit{positive}} distilled trajectories.}
\resizebox{\linewidth}{!}{%
\begin{tabular}{lccccccc}
\toprule
Epoch & Math500 & Minerva & Olympia & AMC & MMLU & MMLU-Pro & BBH \\
\midrule
Base  & 58.40 & 26.84 & 26.07 & 52.50 & 55.80 & 26.56 & 51.10 \\
5epoch   & 72.80 & 37.13 & 37.19 & 45.00 & 60.95 & 30.34 & 54.69 \\
10epoch  & 75.80 & 38.24 & 40.59 & 65.00 & 64.06 & 32.50 & 61.62 \\
15epoch  & 77.20 & 36.76 & 41.93 & 55.00 & 60.81 & 32.15 & 59.69 \\
20epoch  & 78.00 & 36.76 & 41.78 & 57.50 & 61.03 & 32.70 & 60.58 \\
\bottomrule
\end{tabular}}
\end{subtable}\hfill
\begin{subtable}[t]{0.485\textwidth}
\centering
\caption{Qwen2.5-7B is fine-tuned on the \textbf{math reasoning} dataset using \textbf{\textit{negative}} distilled trajectories.}
\resizebox{\linewidth}{!}{%
\begin{tabular}{lccccccc}
\toprule
Epoch & Math500 & Minerva & Olympia & AMC & MMLU & MMLU-Pro & BBH \\
\midrule
Base  & 58.40 & 26.84 & 26.07 & 52.50 & 55.80 & 26.56 & 51.10 \\
5epoch   & 71.20 & 31.99 & 31.56 & 47.50 & 62.58 & 44.04 & 56.28 \\
10epoch  & 77.20 & 34.93 & 39.26 & 50.00 & 71.39 & 52.14 & 69.49 \\
15epoch  & 78.60 & 39.71 & 38.37 & 52.50 & 72.10 & 52.24 & 71.09 \\
20epoch  & 77.60 & 40.44 & 38.37 & 57.50 & 73.11 & 53.74 & 71.73 \\
\bottomrule
\end{tabular}}
\end{subtable}

\vspace{0.9em}

\begin{subtable}[t]{0.485\textwidth}
\centering
\caption{Qwen2.5-7B is fine-tuned on the \textbf{general reasoning} dataset using \textbf{\textit{positive}} distilled trajectories.}
\resizebox{\linewidth}{!}{%
\begin{tabular}{lccccccc}
\toprule
Epoch & Math500 & Minerva & Olympia & AMC & MMLU & MMLU-Pro & BBH \\
\midrule
Base  & 58.40 & 26.84 & 26.07 & 52.50 & 55.80 & 26.56 & 51.10 \\
5epoch   & 72.00 & 36.76 & 37.33 & 47.50 & 73.62 & 50.61 & 64.05 \\
10epoch  & 74.60 & 37.50 & 41.48 & 55.00 & 73.79 & 53.32 & 69.73 \\
15epoch  & 72.00 & 37.50 & 39.26 & 50.00 & 74.11 & 53.91 & 68.34 \\
20epoch  & 74.40 & 37.50 & 39.85 & 50.00 & 73.42 & 53.22 & 68.23 \\
\bottomrule
\end{tabular}}
\end{subtable}\hfill
\begin{subtable}[t]{0.485\textwidth}
\centering
\caption{Qwen2.5-7B is fine-tuned on the \textbf{general reasoning} dataset using \textbf{\textit{negative}} distilled trajectories.}
\resizebox{\linewidth}{!}{%
\begin{tabular}{lccccccc}
\toprule
Epoch & Math500 & Minerva & Olympia & AMC & MMLU & MMLU-Pro & BBH \\
\midrule
Base  & 58.40 & 26.84 & 26.07 & 52.50 & 55.80 & 26.56 & 51.10 \\
5epoch   & 76.80 & 36.76 & 37.78 & 47.50 & 71.09 & 43.99 & 66.00 \\
10epoch  & 76.80 & 37.87 & 40.30 & 52.50 & 71.43 & 45.87 & 68.84 \\
15epoch  & 76.80 & 37.13 & 41.48 & 55.00 & 71.30 & 44.62 & 69.30 \\
20epoch  & 77.00 & 37.13 & 42.07 & 60.00 & 71.23 & 45.79 & 69.46 \\
\bottomrule
\end{tabular}}
\end{subtable}

\vspace{0.9em}

\begin{subtable}[t]{0.485\textwidth}
\centering
\caption{Qwen2.5-32B is fine-tuned on the \textbf{math reasoning} dataset using \textbf{\textit{positive}} distilled trajectories.}
\resizebox{\linewidth}{!}{%
\begin{tabular}{lccccccc}
\toprule
Epoch & Math500 & Minerva & Olympia & AMC & MMLU & MMLU-Pro & BBH \\
\midrule
Base  & 63.20 & 34.19 & 26.52 & 35.00 & 68.34 & 39.80 & 58.65 \\
5epoch   & 90.20 & 49.63 & 59.11 & 85.00 & 76.53 & 46.77 & 78.04 \\
10epoch  & 92.60 & 50.00 & 60.44 & 85.00 & 78.63 & 51.67 & 79.01 \\
15epoch  & 93.00 & 48.53 & 62.07 & 90.00 & 78.72 & 51.99 & 80.57 \\
20epoch  & 91.40 & 50.74 & 60.89 & 85.00 & 79.01 & 54.31 & 80.61 \\
\bottomrule
\end{tabular}}
\end{subtable}\hfill
\begin{subtable}[t]{0.485\textwidth}
\centering
\caption{Qwen2.5-32B is fine-tuned on the \textbf{math reasoning} dataset using \textbf{\textit{negative}} distilled trajectories.}
\resizebox{\linewidth}{!}{%
\begin{tabular}{lccccccc}
\toprule
Epoch & Math500 & Minerva & Olympia & AMC & MMLU & MMLU-Pro & BBH \\
\midrule
Base  & 63.20 & 34.19 & 26.52 & 35.00 & 68.34 & 39.80 & 58.65 \\
5epoch   & 88.40 & 45.22 & 52.30 & 85.00 & 83.07 & 68.23 & 83.55 \\
10epoch  & 92.20 & 51.10 & 57.93 & 85.00 & 85.14 & 73.75 & 84.22 \\
15epoch  & 91.20 & 50.74 & 57.33 & 90.00 & 85.02 & 73.48 & 84.62 \\
20epoch  & 92.20 & 50.74 & 58.37 & 95.00 & 85.47 & 73.53 & 84.51 \\
\bottomrule
\end{tabular}}
\end{subtable}

\vspace{0.9em}

\begin{subtable}[t]{0.485\textwidth}
\centering
\caption{Qwen2.5-32B is fine-tuned on the \textbf{general reasoning} dataset using \textbf{\textit{positive}} distilled trajectories.}
\resizebox{\linewidth}{!}{%
\begin{tabular}{lccccccc}
\toprule
Epoch & Math500 & Minerva & Olympia & AMC & MMLU & MMLU-Pro & BBH \\
\midrule
Base  & 63.20 & 34.19 & 26.52 & 35.00 & 68.34 & 39.80 & 58.65 \\
5epoch   & 84.60 & 44.85 & 52.00 & 62.50 & 82.10 & 66.54 & 80.03 \\
10epoch  & 86.60 & 46.69 & 55.70 & 75.00 & 81.14 & 67.01 & 80.69 \\
15epoch  & 85.00 & 47.06 & 56.59 & 75.00 & 81.73 & 68.33 & 81.73 \\
20epoch  & 85.20 & 46.69 & 56.15 & 75.00 & 81.97 & 68.54 & 81.60 \\
\bottomrule
\end{tabular}}
\end{subtable}\hfill
\begin{subtable}[t]{0.485\textwidth}
\centering
\caption{Qwen2.5-32B is fine-tuned on the \textbf{math reasoning} dataset using \textbf{\textit{negative}} distilled trajectories.}
\resizebox{\linewidth}{!}{%
\begin{tabular}{lccccccc}
\toprule
Epoch & Math500 & Minerva & Olympia & AMC & MMLU & MMLU-Pro & BBH \\
\midrule
Base  & 63.20 & 34.19 & 26.52 & 35.00 & 68.34 & 39.80 & 58.65 \\
5epoch   & 85.00 & 44.49 & 51.26 & 77.50 & 78.74 & 57.48 & 79.09 \\
10epoch  & 87.20 & 46.30 & 54.52 & 75.00 & 79.01 & 60.43 & 80.88 \\
15epoch  & 86.40 & 47.79 & 55.70 & 65.00 & 77.77 & 57.14 & 79.97 \\
20epoch  & 86.40 & 47.06 & 56.89 & 72.50 & 77.99 & 58.34 & 80.71 \\
\bottomrule
\end{tabular}}
\end{subtable}
\caption{\textbf{Checkpoint evaluation across SFT epochs with distilled reasoning trajectories.}
We report performance at 5, 10, 15, and 20 epochs. Each row corresponds to a model size and training dataset, and each row contains two subtables that compare training on \textit{\textit{positive}} (left) versus \textit{negative} (right) distilled trajectories. Columns in each subtable correspond to benchmarks. Rows correspond to training epochs, with `Base' denoting the model before SFT.
}
\label{tab:epoch_sweep_pos_neg}
\end{table*}

Table~\ref{tab:epoch_sweep_pos_neg} compares intermediate checkpoints (epochs 5--20) for Qwen2.5-7B and 32B. Across settings, \textit{negative-trajectory} SFT consistently outperforms the base model, yielding gains comparable to its \textit{positive} counterpart while often matching or exceeding it on OOD benchmarks. This confirms that \textit{negatives} provide structured supervision rather than noise.


\begin{figure*}[htbp]
    \centering
    \scalebox{0.75}{ 
    \begin{minipage}{\textwidth}
        \begin{subfigure}{0.48\textwidth}
            \centering
            \includegraphics[width=\linewidth]{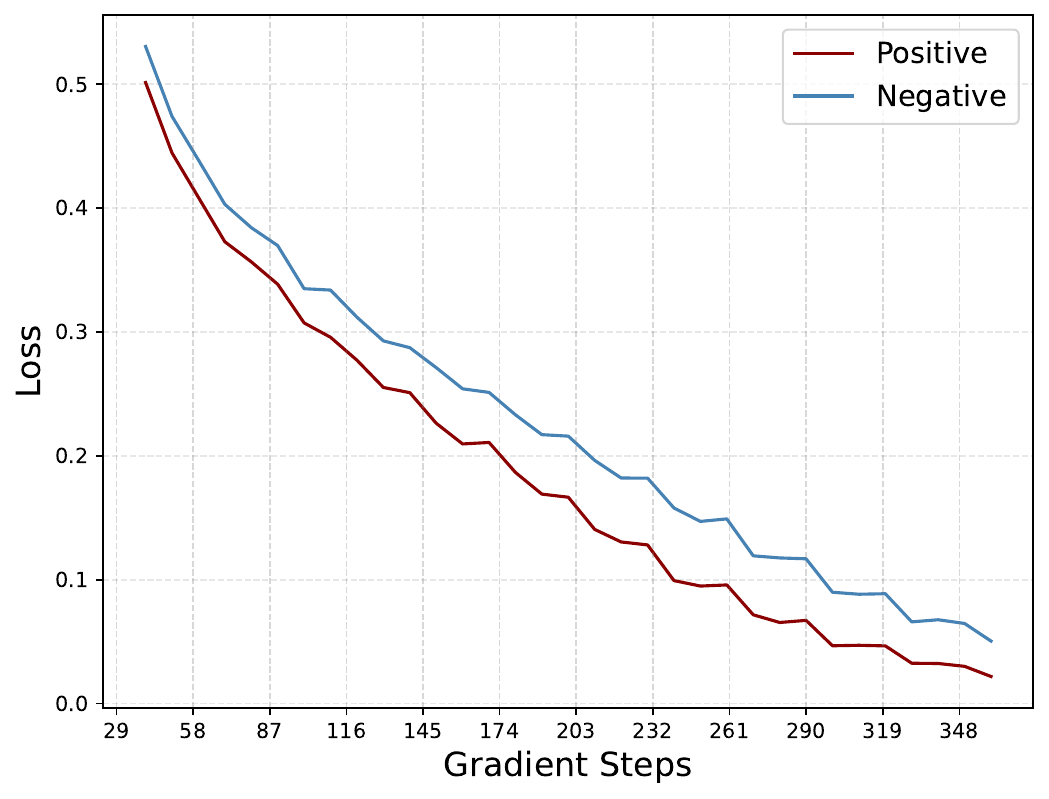}
            \caption{Qwen2.5-3B on OpenMathReasoning}
            \label{fig:1}
        \end{subfigure}
        \begin{subfigure}{0.48\textwidth}
            \centering
            \includegraphics[width=\linewidth]{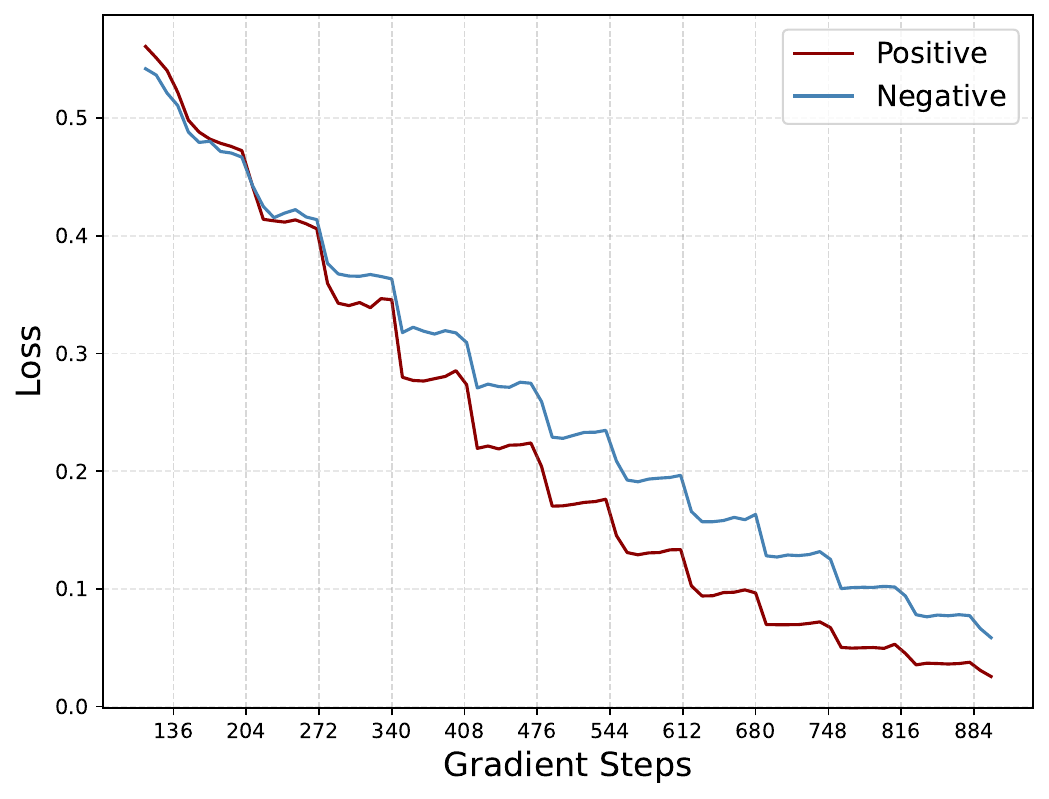}
            \caption{Qwen2.5-3B on MMLU}
            \label{fig:2}
        \end{subfigure}
        
        \begin{subfigure}{0.48\textwidth}
            \centering
            \includegraphics[width=\linewidth]{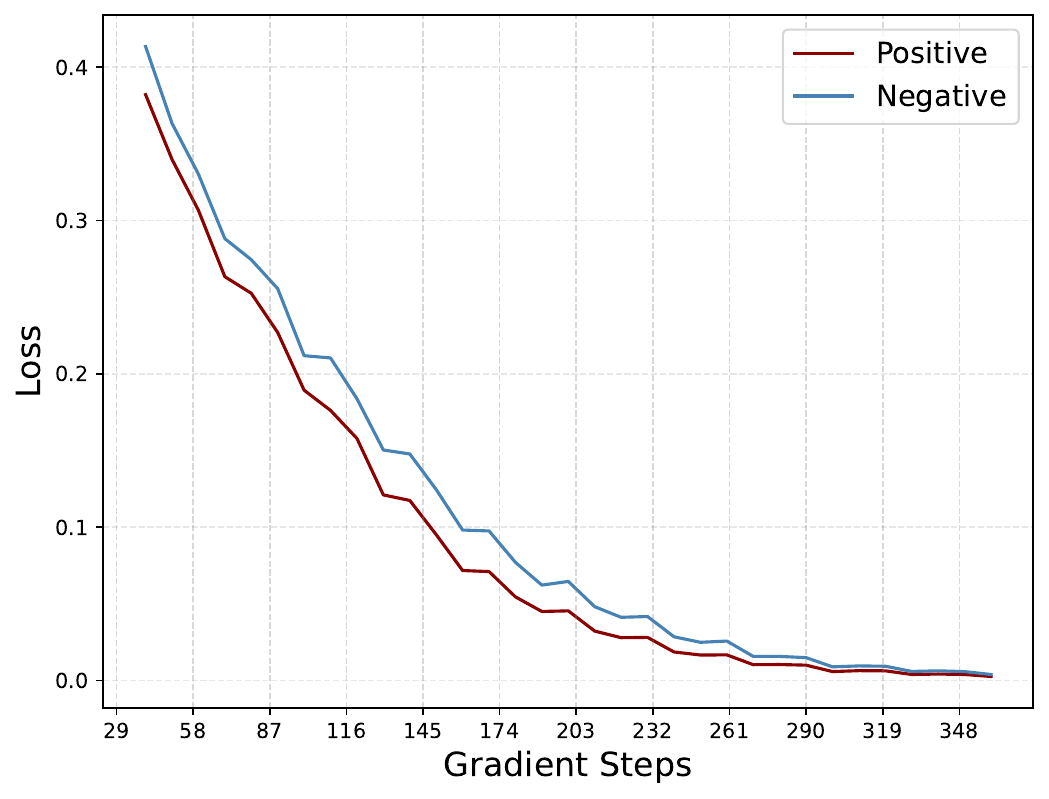}
            \caption{Qwen2.5-7B on OpenMathReasoning}
            \label{fig:3}
        \end{subfigure}
        \begin{subfigure}{0.48\textwidth}
            \centering
            \includegraphics[width=\linewidth]{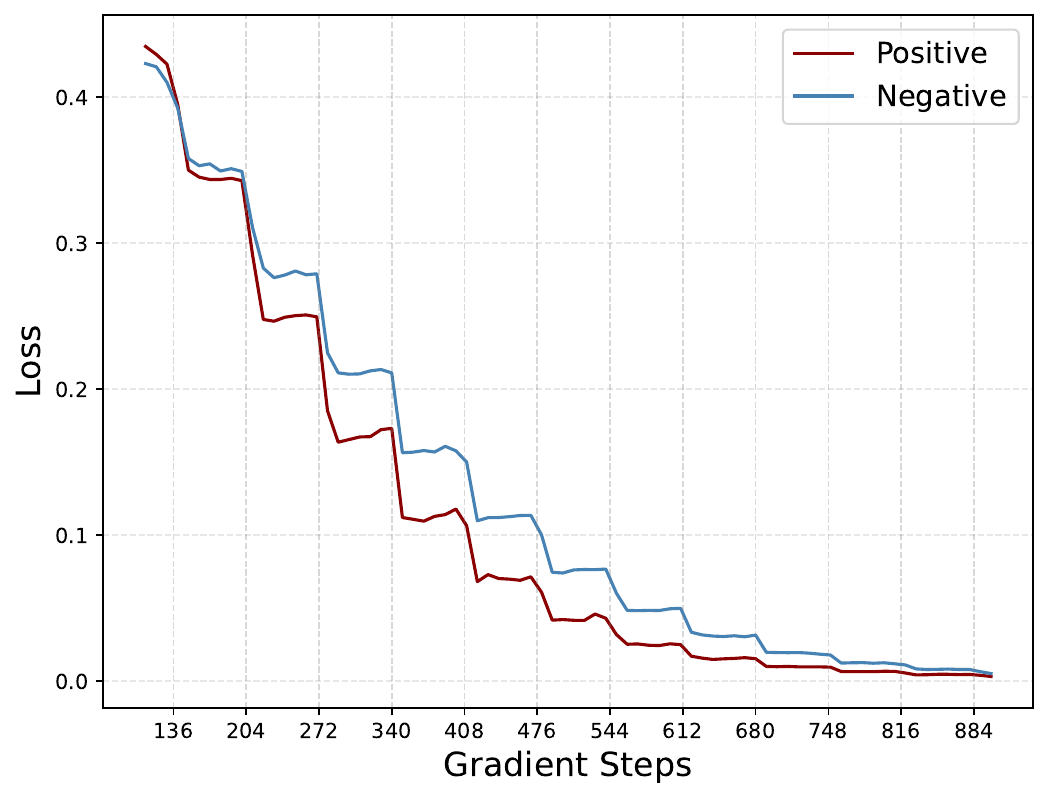}
            \caption{Qwen2.5-7B on MMLU}
            \label{fig:4}
        \end{subfigure}
        
        \begin{subfigure}{0.48\textwidth}
            \centering
            \includegraphics[width=\linewidth]{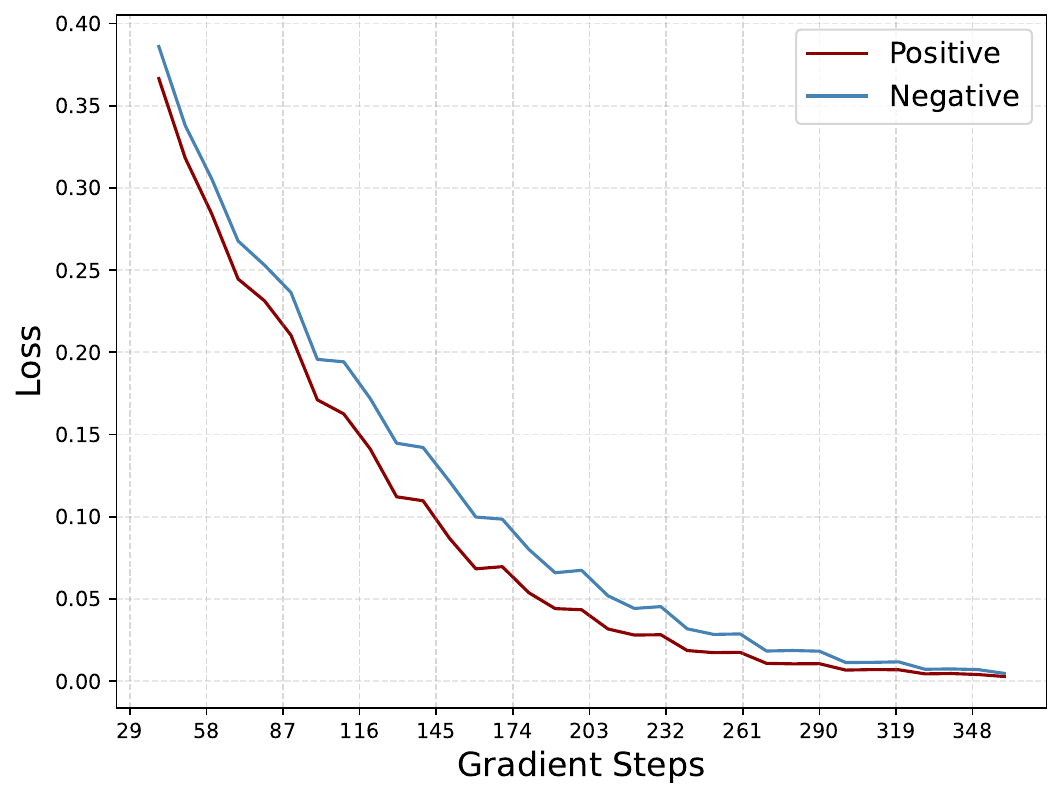}
            \caption{Qwen2.5-14B on OpenMathReasoning}
            \label{fig:5}
        \end{subfigure}
        \begin{subfigure}{0.48\textwidth}
            \centering
            \includegraphics[width=\linewidth]{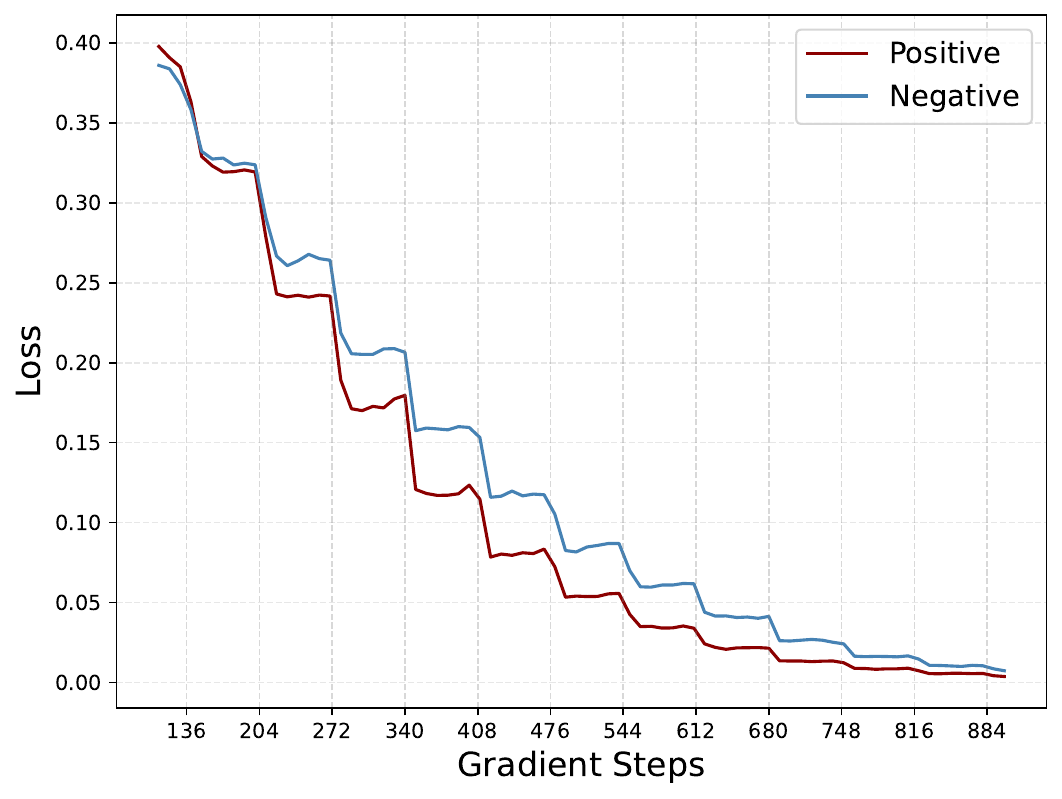}
            \caption{Qwen2.5-14B on MMLU}
            \label{fig:6}
        \end{subfigure}
        
        \begin{subfigure}{0.48\textwidth}
            \centering
            \includegraphics[width=\linewidth]{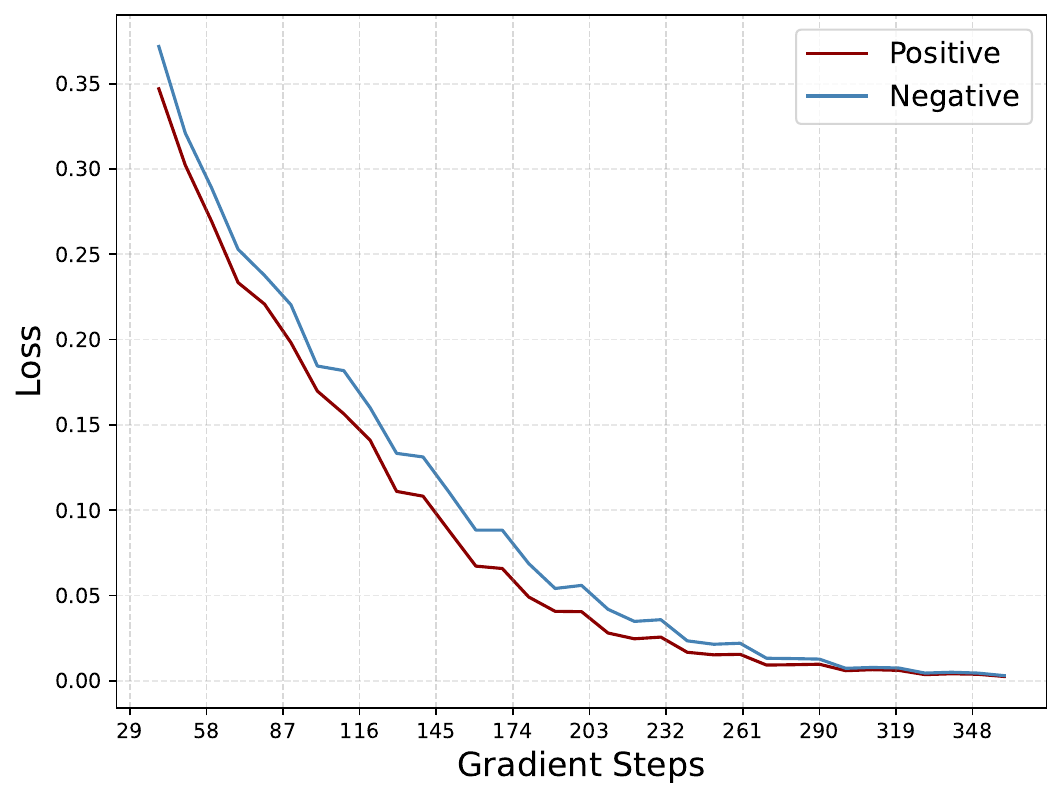}
            \caption{Qwen2.5-32B on OpenMathReasoning}
            \label{fig:7}
        \end{subfigure}
        \begin{subfigure}{0.48\textwidth}
            \centering
            \includegraphics[width=\linewidth]{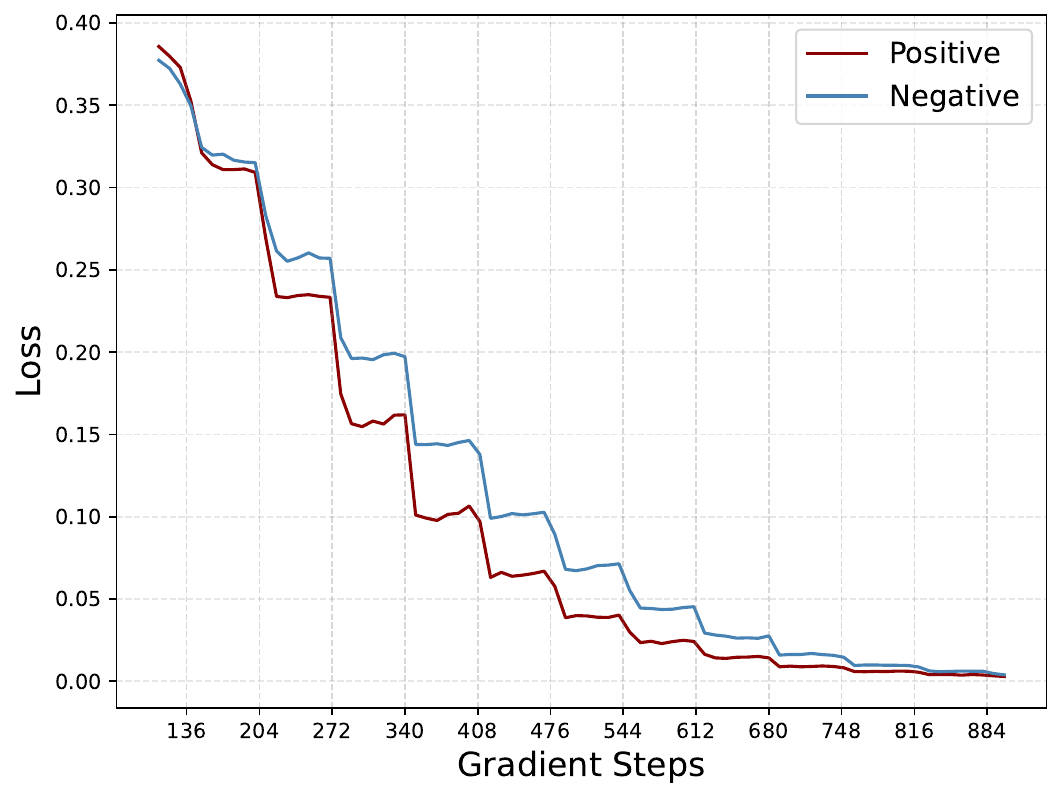}
            \caption{Qwen2.5-32B on MMLU}
            \label{fig:8}
        \end{subfigure}
        
        \begin{subfigure}{0.48\textwidth}
            \centering
            \includegraphics[width=\linewidth]{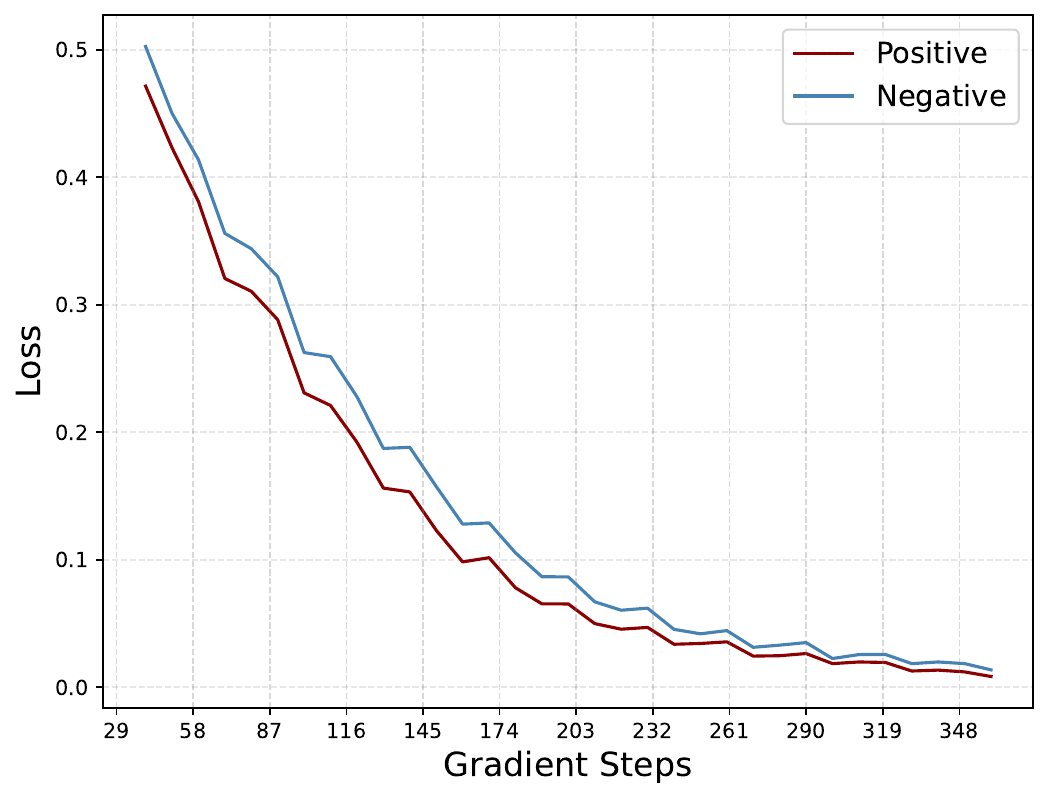}
            \caption{ Llama3.1-8B on OpenMathReasoning}
            \label{fig:9}
        \end{subfigure}
        \begin{subfigure}{0.48\textwidth}
            \centering
            \includegraphics[width=\linewidth]{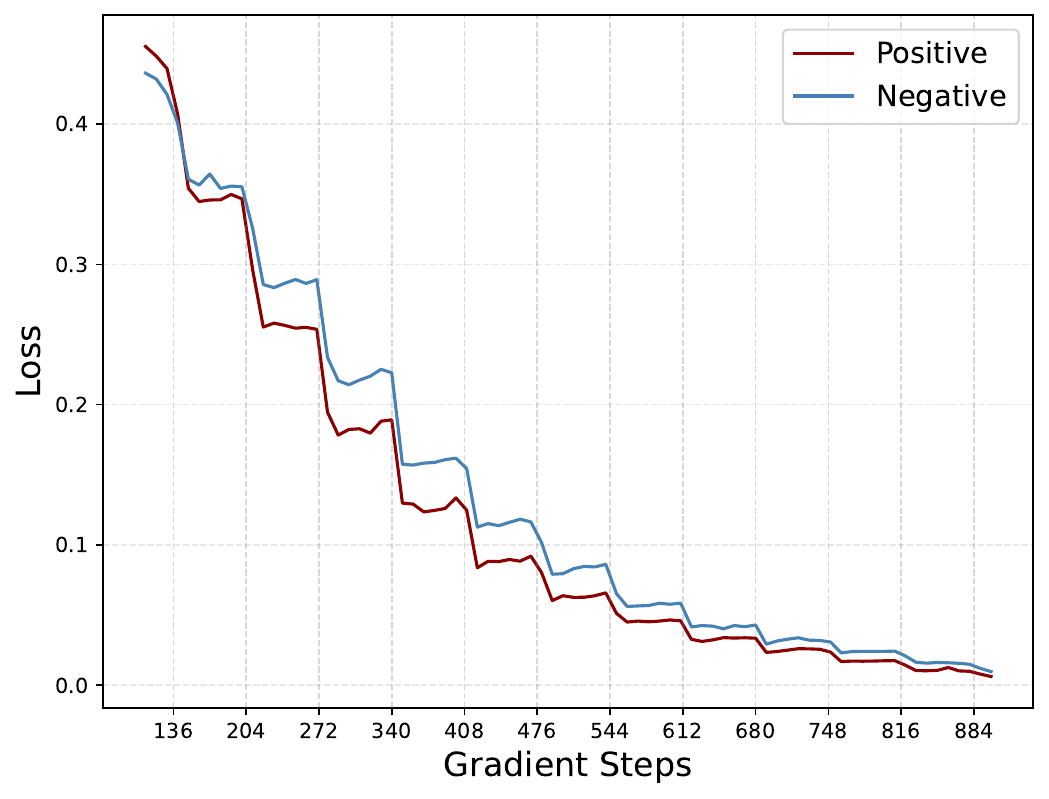}
            \caption{Llama3.1-8B on MMLU}
            \label{fig:10}
        \end{subfigure}
    \end{minipage}}
    
    \caption{Training loss of Qwen2.5 models and Llama3.1-8B on OpenMathReasoning (left) and MMLU (right). Losses drop across epochs, with the \textit{positive} setting converging faster than the \textit{negative}.}
    \label{fig:all}
\end{figure*}


\subsection{Negatives Are Frequently Upweighted by \ours{}}
\label{app:neg_ratio}

Figure~\ref{fig:neg_ratio_both} reports the fraction of \textit{negatives} among the most upweighted examples during \ours{} training. We fine-tune Qwen2.5-3B on Math and MMLU using a shuffled mixture of \textit{positives} and \textit{negatives}, where responses are distilled from Qwen3-8B and labels are determined by final-answer matching. At each optimization step, we select the example with the largest upweighting signal and compute, within each epoch, the proportion of \textit{negatives} among these selections. The fraction stays above 50\% for most epochs, peaks around 75\%--80\% early in training, and then gradually approaches 50\%. This aligns with the design of \ours{}, which emphasizes samples with small inter-epoch loss reduction, a behavior more common among \textit{negatives}.

\begin{figure}[htbp]
    \centering
    \includegraphics[width=\columnwidth]{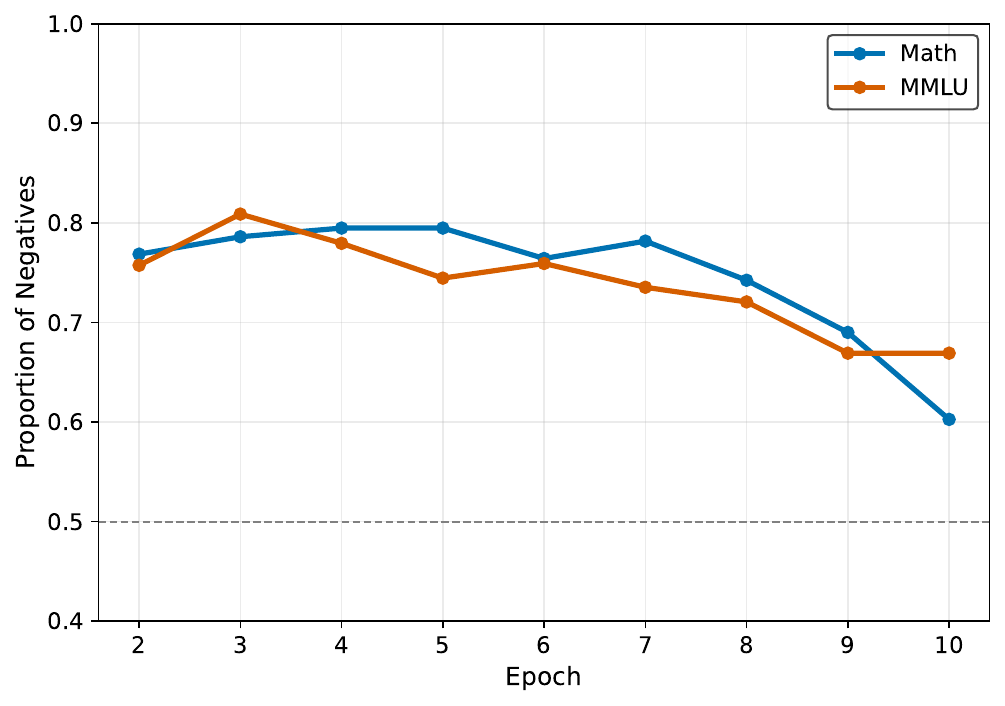}
    \caption{Fraction of \textit{negatives} among stepwise highest-weight samples across epochs for Math and MMLU training.}
    \label{fig:neg_ratio_both}
\end{figure}


\subsection{Pass@k under OOD Evaluation}
\label{app:passk}

We evaluate pass@k ($k\in\{4,8,16,32\}$) averaged over three OOD benchmarks per setting (OpenMath: {BBH, ACPBench, HeadQA}; MMLU: {Olympia, ACPBench, HeadQA}). As shown in Figures~\ref{fig:passk_openmath} and~\ref{fig:passk_mmlu}, \textit{negative}-trained models consistently achieve higher pass@k across all $k$. This superior multi-sample efficiency confirms that \textit{negatives} promote broader reasoning exploration and provide a stronger base policy for subsequent RL.
\begin{figure}[h]
  \centering
  \includegraphics[width=\linewidth]{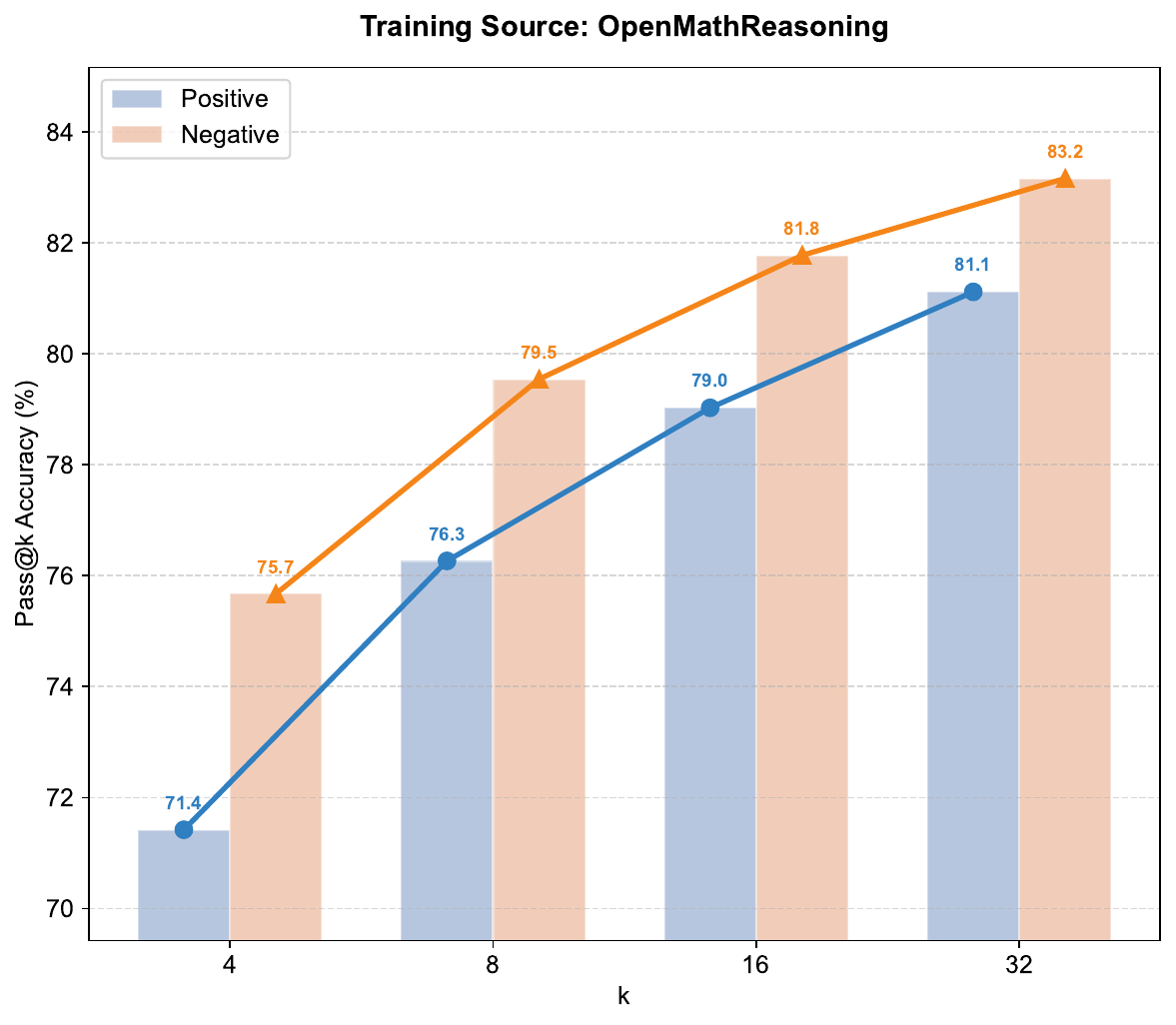}
  \caption{OOD pass@k for models trained on OpenMathReasoning under \textit{positive-only} vs.\ \textit{negative-only} SFT. Results are averaged over BBH, ACPBench, and HeadQA.}
  \label{fig:passk_openmath}
\end{figure}

\begin{figure}[h]
  \centering
  \includegraphics[width=\linewidth]{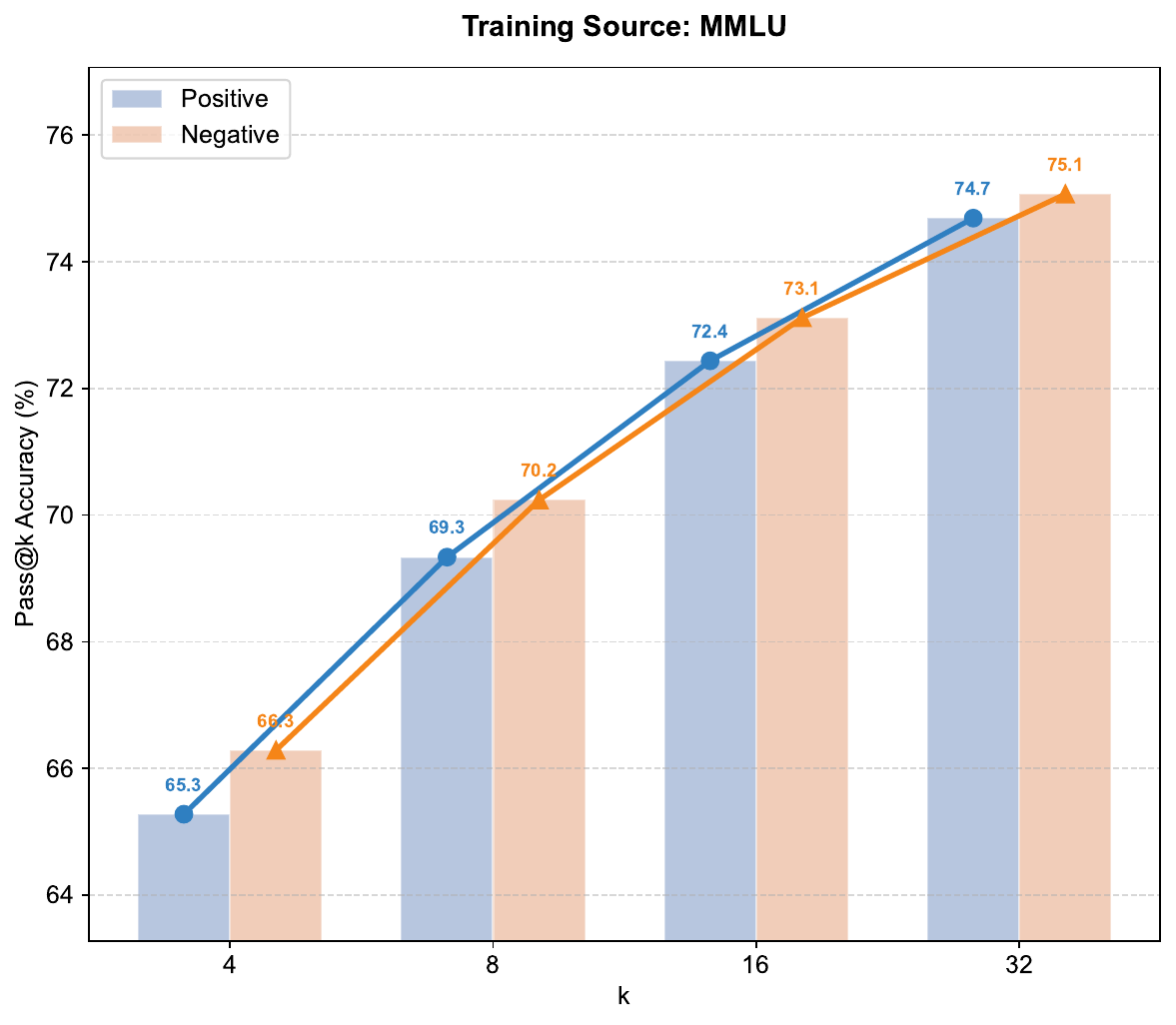}
  \caption{OOD pass@k for models trained on MMLU under \textit{positive-only} vs.\ \textit{negative-only} SFT. Results are averaged over Olympia, ACPBench, and HeadQA.}
  \label{fig:passk_mmlu}
\end{figure}

\subsection{Prompt for Categorize Negative Samples}\label{app:prompt}
We design a structured prompt to categorize each erroneous reasoning trajectory into a fine-grained error class. The classification framework contains 9 primary categories and 22 sub-categories. The full classification schema and the prompt used for categorization are shown in Figure~\ref{fig:prompt}.
\begin{figure*}[htbp!]
\centering
\small

\begin{tcolorbox}[
colback=blue!3!white,
colframe=blue!55!black,
  coltitle=white,
  title=\textbf{Prompt for Categorizing Negative Samples},
  fonttitle=\bfseries
]
\begin{lstlisting}[basicstyle=\ttfamily\small, breaklines=true]
You are an expert AI assistant tasked with identifying the single,
most specific error category from the list below.

Error Category List:
- Primary_category: Understanding Errors
  - sub_category: Problem Misunderstanding, Conceptual Misunderstanding
- Primary_category: Knowledge Errors
  - sub_category: Factual Error, Theorem Error, Definition Error
- Primary_category: Logical Errors
  - sub_category: Strategy Error, Reasoning Error, Premise Error, Consistency Error
- Primary_category: Calculation Errors
  - sub_category: Numerical Error, Formula Error, Parameter Error, Unit Error
- Primary_category: Programming Errors
  - sub_category: Syntax Error, Function Error, Data Type Error
- Primary_category: Formal Errors
  - sub_category: Symbol Error, Formatting Error
- Primary_category: Completeness Errors
  - sub_category: Boundary Omission
- Primary_category: Special Cases
  - sub_category: Reflection Error, Summary Error, Hallucination, Redundancy
- Primary_category: Evaluation System Errors
  - sub_category: Incorrect Ground Truth, Correct Answer Parsing Error

Data for Analysis:
- Question: {question}
- Ground Truth Answer: {groundtruth}
- Model's Reasoning Process (to be analyzed): {model_reasoning}

CRITICAL INSTRUCTION:
Analyze the provided reasoning process. Your response MUST be ONLY a single,
raw JSON object with the keys "sub_category" and "analysis". Do not include any
other text, explanations, apologies, or markdown formatting.

Example of a perfect response:
{
  "sub_category": "Premise Error",
  "analysis": "The model incorrectly assumed that all bicycles use plastic
               squares for identification, which is a flawed premise not
               supported by the question's context."
}
\end{lstlisting}
\end{tcolorbox}
\caption{Prompt used for categorizing \text{negative} reasoning samples into predefined error subcategories.}
\label{fig:prompt}
\end{figure*}


\subsection{Case Study of Negative Samples}
As discussed in Section~\ref{sec:analysis_infer}, \textit{negative} trajectories exhibit higher entropy than positives ones on certain reasoning tokens and transition words. For illustration, we select one case and highlight the high-entropy segments. The results in Figure~\ref{fig:case_study_neg_path} show that \textit{negatives} contain substantially more such reasoning-related high-entropy fragments than \textit{positives}.
%
%
\begin{figure*}[htbp!]
\centering
\small

\begin{tcolorbox}[
  colback=gray!4!white,
  colframe=gray!85!black,
  coltitle=white,
  title=\textbf{Problem},
  fonttitle=\bfseries
]
The graph shows the total distance Sam drove from 6 a.m.\ to 11 a.m. How many miles per hour is the car's average speed for the period from 6 a.m.\ to 11 a.m.?  

\medskip

unitsize(0.2inch);
draw((0,0)--(5.5,0));
draw((0,0)--(0,8.5));
draw((1,0)--(1,8.5));
draw((2,0)--(2,8.5));
draw((3,0)--(3,8.5));
draw((4,0)--(4,8.5));
draw((5,0)--(5,8.5));
draw((0,1)--(5.5,1));
draw((0,8)--(5.5,8));
draw((0,7)--(5.5,7));
draw((0,6)--(5.5,6));
draw((0,5)--(5.5,5));
draw((0,4)--(5.5,4));
draw((0,3)--(5.5,3));
draw((0,2)--(5.5,2));
draw((0,0)--(1,2)--(2,3)--(3,5)--(4,6)--(5,8));
dot((0,0));
dot((1,2));
dot((2,3));
dot((3,5));
dot((4,6));
dot((5,8));
label("6",(0,-0.5),S);
label("7",(1,-0.5),S);
label("8",(2,-0.5),S);
label("9",(3,-0.5),S);
label("10",(4,-0.5),S);
label("11",(5,-0.5),S);
label("0",(-0.5,0),W);
label("40",(-0.5,2),W);
label("80",(-0.5,4),W);
label("120",(-0.5,6),W);
label("160",(-0.5,8),W);
label("Time of Day (a.m.)",(2.7,-2),S);
label("Total distance",(-0.5,9),N);

\end{tcolorbox}

\vspace{0.1\baselineskip}

\begin{tcolorbox}[
  colback=teal!2!white,
  colframe=teal!80!black,
  coltitle=white,
  title=\textbf{Model's Thinking Trajectory and Answer},
  fonttitle=\bfseries
]

Let me start with the total time. The period is from 6 a.m.\ to 11 a.m. Let me count the hours. From 6 to 7 is 1 hour, 7 to 8 is another, 8 to 9, 9 to 10, and 10 to 11. That's 5 hours total. Wait, is that right? Let me check again. From 6 to 11 is 5 hours? Wait, 11 minus 6 is 5, yes. So total time is 5 hours. Okay, that seems straightforward.

Now, the total distance. The graph shows the total distance driven at different times. The Asymptote code draws a graph with time on the x-axis and total distance on the y-axis. The points given are at time 6 a.m.\ (which is the first point, x = 0), then at 7 a.m.\ (x = 1), 8 a.m.\ (x = 2), etc., up to 11 a.m.\ (x = 5). The corresponding distances are given as y-values. Let me check the Asymptote code again to see the coordinates.

Looking at the Asymptote code, the graph is drawn from $(0,0)$ to $(1,2)$ to $(2,3)$ to $(3,5)$ to $(4,6)$ to $(5,8)$. The labels on the y-axis are 0, 40, 80, 120, 160. So each unit on the y-axis corresponds to 40 miles? Let me confirm. The labels are at $y = 0, 2, 4, 6, 8$. The labels next to them are 0, 40, 80, 120, 160. So, each vertical unit is 40 miles. Therefore, the y-coordinate multiplied by 40 gives the total distance in miles.

So, for example:
\begin{itemize}
\item At 6 a.m.\ (x = 0), the total distance is 0 miles.
\item At 7 a.m.\ (x = 1), the total distance is 2 units, which is $2 \times 40 = 80$ miles.
\item At 8 a.m.\ (x = 2), it's 3 units, so $3 \times 40 = 120$ miles.
\item At 9 a.m.\ (x = 3), 5 units, which is $5 \times 40 = 200$ miles.
\item At 10 a.m.\ (x = 4), 6 units, which is $6 \times 40 = 240$ miles.
\item At 11 a.m.\ (x = 5), 8 units, which is $8 \times 40 = 320$ miles.
\end{itemize}

The question is about the total distance from 6 a.m.\ to 11 a.m. The graph is showing cumulative distance. So at 11 a.m., the total distance is 320 miles. Therefore, the total distance driven is 320 miles.

If the graph is cumulative, then the total distance at 11 a.m.\ is the total distance driven from 6 a.m.\ to 11 a.m., so that would be 320 miles. Therefore, total distance is 320 miles, total time is 5 hours. Therefore, average speed is $320 \div 5 = 64$ miles per hour.

Let me check if the total distance is indeed 320. The last point is at 11 a.m., which is $x = 5, y = 8$. Since each $y$ unit is 40 miles, $8 \times 40 = 320$. Yes. So that seems correct.

Alternatively, maybe the question is trying to trick me? Suppose someone might think that the total distance is the sum of all the segments, but because the graph is cumulative, the total distance at 11 a.m.\ is already the total distance driven from 6 a.m.\ to 11 a.m.

We can also confirm by looking at increments: From 6--7: $80$ miles, From 7--8: $120 - 80 = 40$ miles. From 8--9: $200 - 120 = 80$ miles. From 9--10: $240 - 200 = 40$ miles. From 10--11: $320 - 240 = 80$ miles
Summing gives $80 + 40 + 80 + 40 + 80 = 320$ miles, consistent with the final value.

So average speed is total distance divided by total time. Total time is 5 hours. $320 \div 5 = 64$, so 64 miles per hour.

The time from 6 a.m.\ to 11 a.m.\ is indeed 5 hours. The x-axis is labeled from 6 a.m.\ to 11 a.m., with each unit representing one hour, so that checks out.

Therefore, I think my answer is correct.

\textbf{Final Answer}: the car's average speed is $\boxed{64}$ miles per hour.

\end{tcolorbox}

\caption{Case study of a negative trajectory from the OpenMathReasoning training dataset. The model misinterprets the problem, but its subsequent step-by-step reasoning and formula derivations remain structurally correct.}
\label{fig:case_study_neg_path}
\end{figure*}


\subsection{Case Study of Samples Generated by Various Models}
To qualitatively evaluate the differences in reasoning behavior, we provide a comparative case study in Figure~\ref{fig:case_study}, contrasting trajectories from $M_{\text{pos}}$ and $M_{\text{neg}}$. $M_{\text{neg}}$ tends to exhibit more frequent use of discourse and hesitation tokens(e.g. ``wait'' ,``but''), particularly when encountering complex reasoning steps. These qualitative observations align with the token distribution analysis in Figure~\ref{fig:freq_entropy}, confirming that $M_{\text{neg}}$ allocates a larger portion of its generation budget to connective exploration. By maintaining multiple plausible continuations instead of committing prematurely to a single path, $M_{\text{neg}}$ demonstrates a more exhaustive search of the reasoning space before finalizing its response.
%
%
\begin{figure*}[htbp!]
\centering
\small
\begin{tcolorbox}[
  colback=gray!4!white,
  colframe=gray!85!black,
  coltitle=white,
  title=\textbf{An Example from MATH-500},
  fonttitle=\bfseries
]
\textbf{Problem}: Denali and Nate work for a dog walking business and are paid for each dog they walk. Denali is responsible for $16$ dogs and Nate is responsible for $12$ dogs. Under the company's new policy, they will be assigned or unassigned new dogs in groups of $x$ dogs. The ratio of Denali's pay to Nate's pay would be the same if Denali started walking $4x$ more dogs and Nate stayed at $12$ dogs or if $x$ of Nate's dogs were reassigned to Denali. Find $x$ if $x\neq0$. \textbf{Answer}: \boxed{5}
\end{tcolorbox}

\noindent
\begin{minipage}[t]{0.49\textwidth}
  \begin{tcolorbox}[
    colback=teal!2!white, 
    colframe=teal!80!black, 
    coltitle=white,
    title=\textbf{$M_{pos}$ Thinking Trajectories},
    fonttitle=\bfseries,
    equal height group=mathcompare,
  ]
  Okay, let me try to figure out this problem...

So, setting up the equation...

Let me compute both sides...

\vspace{1\baselineskip}
...would be the same under two different scenarios. Let me parse that...

\vspace{1\baselineskip}
So, solution is x = 0...

\vspace{1\baselineskip}
But the problem says ...

\vspace{1\baselineskip}
Let me check again...

\vspace{1\baselineskip}
x = 0 is correct...

The final answer is \(\boxed{0}\).
  \end{tcolorbox}
\end{minipage}
\hfill
\begin{minipage}[t]{0.49\textwidth}
  \begin{tcolorbox}[
    colback=teal!2!white, 
    colframe=teal!80!black, 
    coltitle=white,
    title=\textbf{$M_{neg}$ Thinking Trajectories},
    fonttitle=\bfseries,
    equal height group=mathcompare,
  ]
Okay, let me try to figure out this problem...

The key is that...

Now, we can solve this...

..., so the answer is x = 3.

Wait, let me check...

it is wrong...

\vspace{1\baselineskip}
\textit{Alternatively, maybe there's a different way to interpret the problem?...}

Maybe I missed something...

\vspace{1\baselineskip}
\textit{Alternatively, maybe there's a different interpretation? Let me check...}

Wait, no. The problem says...

..., so the answer is 5.

But, Let me check ...

Therefore, x = 5 is correct...

The final answer is \(\boxed{5}\).  
  \end{tcolorbox}
\end{minipage}

\caption{Case study of thinking trajectories for $M_{pos}$ and $M_{neg}$ on the same question. }
\label{fig:case_study}
\end{figure*}


\end{document}